%% file: main.tex
\documentclass[hidelinks,onefignum,onetabnum]{siamart250211} 

\usepackage{amsmath}
\usepackage{amssymb}
\usepackage{bm}
\usepackage{mathdots}
\usepackage{latexsym}
\usepackage{graphicx}
\usepackage{pgfplots}
\usepackage{hyperref}

\usepackage{color}
\usepackage{fullpage}
\usepackage{verbatim}
\usepackage{algorithm}
\usepackage{algorithmic}
\usepackage{amsopn,url}
\usepackage[caption=false]{subfig}

\def\Pi{\prod}

\def\Xt{\widetilde X}

\def\Vt{\widetilde V}
\def\Uh{\widehat U}
\def\Ut{\widetilde U}
\def\ut{\widetilde u}
\def\vt{\widetilde v}

\def\Uc{\check U}
\def\Xc{\check X}
\def\Vc{\check V}
\def\Sigc{\check \Sigma}

\def\Sigh{\widehat\Sigma}
\def\sigh{\widehat\sigma}

\def\R{\mathbb{R}}
\newcommand{\uni}[1]{{\left\vert\kern-0.25ex\left\vert\kern-0.25ex\left\vert #1  \right\vert\kern-0.25ex\right\vert\kern-0.25ex\right\vert}} 
\newcommand{\unii}[1]{{\vert\kern-0.25ex\vert\kern-0.25ex\vert #1  \vert\kern-0.25ex\vert\kern-0.25ex\vert}} 
\newcommand{\Uniinv}[1]{{\big\vert\kern-0.25ex\big\vert\kern-0.25ex\big\vert #1  \big\vert\kern-0.25ex\big\vert\kern-0.25ex\big\vert}} 
\newcommand{\uniinv}[1]{{\left\vert\kern-0.25ex\left\vert\kern-0.25ex\left\vert #1  \right\vert\kern-0.25ex\right\vert\kern-0.25ex\right\vert}} 

\newcommand{\mbb}{\mathbb}

\renewcommand{\Re}{\mbb{R}}

\newcommand{\norm}[1]{\left\|{#1}\right\|}

\newcommand{\image}{\operatorname{span}}

\newcommand{\diag}{{\rm diag}}

\newtheorem{remark}{{\sc Remark}}[section]

\newcommand{\ignore}[1]{}

\title{SuperPCA: subspace analysis and an efficient algorithm for 
high-dimensional PCA}
\author{
Irina-Beatrice Haas\thanks{Mathematical Institute, University of Oxford, Oxford, OX2 6GG, UK, (\email{irina-beatrice.haas@maths.ox.ac.uk}, \email{yuji.nakatsukasa@maths.ox.ac.uk}).}
\and 
Maike Meier\thanks{Bernoulli Institute for Mathematics, Computer Science and Artificial Intelligence, University of Groningen, 9700 AB, Groningen, The Netherlands (\email{m.meier@rug.nl})}
\and Yuji Nakatsukasa\footnotemark[1]
\and 
Taejun Park\thanks{Institute of Mathematics, EPFL, 1015 Lausanne, Switzerland (\email{taejun.park@epfl.ch}).}}

\begin{document}

\maketitle

\begin{abstract}
    Principal component analysis (PCA) is a fundamental tool to reduce the dimensionality of the data in many applications. PCA finds a few signal directions that contain most of the variability of the data by computing the eigenvectors of the sample covariance matrix. In this work, we focus on the spiked covariance model, in which the data vectors are defined by a few orthogonal signals plus an isotropic Gaussian noise, and our goal is to estimate one or more of the leading signals. Our main theoretical finding is that the subspace spanned by several leading eigenvectors of the sample covariance matrix contains significant information about the desired signals long before the individual eigenvectors converge to the population principal components. To prove this, we derive a posteriori bounds for the angle between the subspace spanned by the desired population signals and the subspace obtained from the sample using perturbation theory for singular vectors. This leads to a new algorithm, SuperPCA (SUbsPace subsamplER PCA), which capitalizes on an approximate eigenspace of the sample covariance matrix to find the leading signals far more efficiently and accurately than classical PCA in the high-dimensional, multi-signal setting. SuperPCA exploits only a small number of subsampled coordinates of the data, which can lead to tremendous savings in data acquisition cost, especially when the signals are approximately sparse. For the same number of measurements, SuperPCA can offer a factor $10$ improvement in accuracy compared to the classical PCA method. 
\end{abstract}



\normalsize

\input{1.Introduction}

\input{2.LitReview_new}

\input{TheoryAngles_new}

\input{SuperPCA_new}

\input{ExperimentsSPCA}

\input{Proofs}

\bibliographystyle{siamplain} 
\bibliography{references}

\end{document}

%% file: 1.Introduction.tex
\section{Introduction}\label{sec:introduction}
Principal Component Analysis (PCA) is an important tool in data analysis and dimension reduction. Given a dataset, PCA seeks to find (orthogonal) directions that best explain the variance in the data.  We consider the popular spiked covariance model introduced in \cite{johnstone2001distribution}.

In this model, we observe $n$ $p$-dimensional measurements of the form
\begin{equation}\label{eq:1.defmodel}
    x_i = \sum_{j=1}^k\sqrt{\beta_j}g_j^iu_j + \sigma \eta_i, \quad i = 1,\dots,n
\end{equation}
where $\sqrt{\beta_j}g_j^iu_j$ are signals and $\sigma \eta_i$ is noise. We have $k(\geq 1)$ $p$-dimensional signals with orthonormal directions $u_1,\dots,u_k$ and respective signal strengths $\beta_1\geq\dots\geq\beta_k>0$. The noise is assumed random normal with intensity $\sigma>0$. The $g_j^i\sim\mathcal{N}(0,1)$ and $\eta_i\sim\mathcal{N}(0, I_p)$ are scalar and vector standard normal random variables, respectively. Each measurement $x_i$ is then independently and identically distributed like $x\sim\mathcal{N}(0, K_{\textrm{pop}})$ with population covariance matrix $K_{\textrm{pop}}$ given by
\begin{equation}\label{eq:1.defsigma1}
    K_{\mathrm{pop}} = \sum_{j=1}^k\beta_j u_j u_j^T + \sigma^2 I_p.
\end{equation}
The population covariance matrix has dominant eigenvectors $u_1,\dots,u_k$ with corresponding eigenvalues $\lambda_1 = \beta_1 + \sigma^2,\dots, \lambda_k = \beta_k + \sigma^2$. We call the vectors $u_j$ the population principal components. The aim of PCA is to approximate the principal components based on samples drawn from $\mathcal{N}(0,K_{\mathrm{pop}})$. Classically, this is done through the eigendecomposition of the sample covariance matrix. Namely, collect the samples in a data matrix $X\in\mathbb{R}^{p\times n}$ of the form $ X = \left[x_1,\dots,x_n\right]$ and define the sample covariance matrix $S_n = XX^T/n$. The main topic of analysis in PCA is how well the dominant eigenvectors of $S_n$, say $\hat{u}_1,
\dots,
\hat{u}_k$, called the sample principal components, approximate the population principal components, i.e., $u_1,\ldots,u_k$.

Instead of the eigenvalue decomposition, the analysis in this paper is based on singular vector perturbation theory, and we will mostly focus on the         singular value decomposition (SVD) 
of the scaled data matrix $X_n = X/\sqrt{n} = \hat{U}\hat{\Sigma}\hat{V}$. The left singular vectors $\hat{U}$ are the same as the eigenvectors of $S_n$. The singular values $\hat{\sigma}_1\geq\dots\geq\hat{\sigma}_{\min(p,n)}\geq 0$ are the square roots of the eigenvalues of $S_n$. Ideally, the dominant singular values of $X_n$ approximate the singular values $\sigma_1,\dots,\sigma_p$ of the square root of $K_{\mathrm{pop}}$. That is, $\sigma_j = \sqrt{\lambda_j} = \sqrt{\beta_j + \sigma^2}$ for $j=1,\dots,k$ and $\sigma_j= \sigma$ for $j>k$.

PCA can be considered in different (asymptotic) regimes. In the classical context the number of samples $n$ tends to infinity and the dimension of the problem $p$ is fixed.  
It is well known that the sample covariance matrix then converges to the population covariance matrix almost surely as $n\to\infty$. As a result, the sample principal components consistently estimate the population principal components ~\cite{anderson1963asymptotic,muirhead1982aspects}. 

However, in modern applications, the dimension of the problem $p$ is often of comparable size to the number of samples $n$, sometimes even larger. In this regime, the sample principal components are generally inconsistent and exhibit the so-called Ben Arous-Baik-P\'{e}ch\'{e} (BBP) phase transition which describes sufficient conditions on the signal to noise ratio and dimensions $n,p$ to detect the signal \cite{baik2005phase,Baik2006,johnstone2001distribution,Paul2007,yao2015sample}. Recent work has therefore focused on finite-sample bounds that quantify the estimation error in the principal components for large but finite $p$ and $n$ \cite{johnstone2018pca,koltchinskii2017normal,Nadler2008,reiss2020nonasymptotic,vaswani2017finite}. These bounds typically depend on spectral gaps between consecutive eigenvalues of the population covariance matrix. 

Our work builds on this finite-sample perspective. Rather than estimating each signal from a sample principal component of the same dimension, we show that a larger sample left singular-subspace of $S_n$ can provide substantially more accurate information about the desired population subspace, especially when several signal strengths are close together.

The theoretical findings have computational and practical implications. Once we find a subspace in which the desired signal lies with sufficient accuracy, a natural goal is to find the signal from the subspace. We introduce an algorithm, SuperPCA, that attempts to do so.

\subsection{Motivating examples}
The reasoning behind considering subspaces of differing dimensions can be illustrated with examples. First, consider a model with three (orthonormal) signals $u_1$, $u_2$, and $u_3$ with respective signal strengths $\beta_1\geq\beta_2\geq\beta_3$. Let $\hat{u}_1$, $\hat{u}_2$, and $\hat{u}_3$ denote the three dominant left singular vectors of the corresponding data matrix $X_n$. We are interested in how well $\hat{u}_1$, $\mathrm{span}\,(\hat{u}_1, \hat{u}_2)$, and $\mathrm{span}\,(\hat{u}_1, \hat{u}_2, \hat{u}_3)$ can approximate the principal component of interest: $u_1$. The relevant quantity, we argue, is the angle between $u_1$ and the above mentioned vector and subspaces. If the sine of this angle is close to 0, then the principal component is well-contained. As such, the sine of the angle can be considered as a measure of distance between $u_1$ and the respective subspace. 

In Figure~\ref{fig:introductionSimple}, we fix $\beta_1 = 75$ and $\beta_3 = 25$, and we vary the strength of the second signal $\beta_2 \in [\beta_3,\beta_1]$. We plot the distance between $u_1$ and the subspaces $\mathrm{span}\,(\hat{u}_1)$, $\mathrm{span}\,(\hat{u}_1, \hat{u}_2)$, and $\mathrm{span}\,(\hat{u}_1, \hat{u}_2, \hat{u}_3)$. This is repeated for (approximations) to different asymptotic and finite-sample regimes. In the second frame, $p$ and $n$ tend to infinity together, and theoretically the limit of the angle $\theta(u_1,\hat{u}_1)$ is only dependent on $n$, $p$, $\sigma$, and $\beta_1$ and independent from $\beta_2$ and $\beta_3$, provided $\beta_1\neq\beta_2$ (see ~\cite[Thm. 4]{Paul2007}). In the same asymptotic regime, yet under the assumption $\beta_1 = \beta_2 > \beta_3$, it is impossible to distinguish between $u_1$ and $u_2$. We can only use $\image(\hat{u}_1,\hat{u}_2)$ to approximate $\image(u_1,u_2)$. This phenomenon is displayed in the figure as $\sin\theta(u_1,\hat{u}_1)$ is independent from $\beta_2$ for a large portion of the graph. The inconsistency for values of $\beta_2$ close to $\beta_1$ can be explained by the finiteness of $p$ and $n$. The third panel aims to approximate the classical asymptotic regime. As $n\to\infty$,  $\sin\theta(u_1,\hat{u}_1)\to 0$. In the figure, the angle is dependent on $\beta_2$ yet consistently small (note the log-scale).

Our main regime of interest is displayed in the first panel of Figure~\ref{fig:introductionSimple}. The distance between $u_1$ and $\hat{u}_1$ deteriorates rapidly as $\beta_2$ increases. Indeed, it is intuitive that, if $\beta_1/\beta_2$ is close to 1, $\hat{u}_2$ would contain information on $u_1$ and $\hat{u}_1$ on $u_2$. Visually, this can be seen in the red, dotted curve, which contains the distance between $u_1$ and $\mathrm{span}\,(\hat{u}_1, \hat{u}_2)$. This distance is consistently small. This experiment shows that when $\beta_1\approx \beta_2$ the error in the leading sample principal component $\hat u_1$ can be large but a larger subspace such as $\image(\hat{u}_1,\hat{u}_2)$ still contains the leading signal $u_1$ with high accuracy. 

Figure~\ref{fig:introductionCos} visualizes the same phenomenon with a slightly different set-up. Here, there are $k=9$ signals with strengths $\beta_1=10,\beta_2=9,\dots,\beta_9 = 2$. We plot the cosine of the angle between $u_1$ and the $j$th sample principal component $\hat{u}_j$, which is close to 1 when the distance between $u_1$ and $\hat{u}_j$ is small. In either asymptotic regime (bottom two frames), we see that $\hat{u}_1$ is very close to $u_1$ and that all further sample principal components $\hat{u}_2,\hat{u}_3,\dots$ contain hardly any components in the direction of $u_1$. However, in the finite sample regime in the top panel, trailing sample components contain large components in the direction of $u_1$. It is immediately clear that, in this context, $\mathrm{span}\,(\hat{u}_1, \hat{u}_2)$ contains much more information on $u_1$ than just $\hat{u}_1$.

\begin{figure}
    \centering
    \includegraphics[width = 0.8\linewidth]{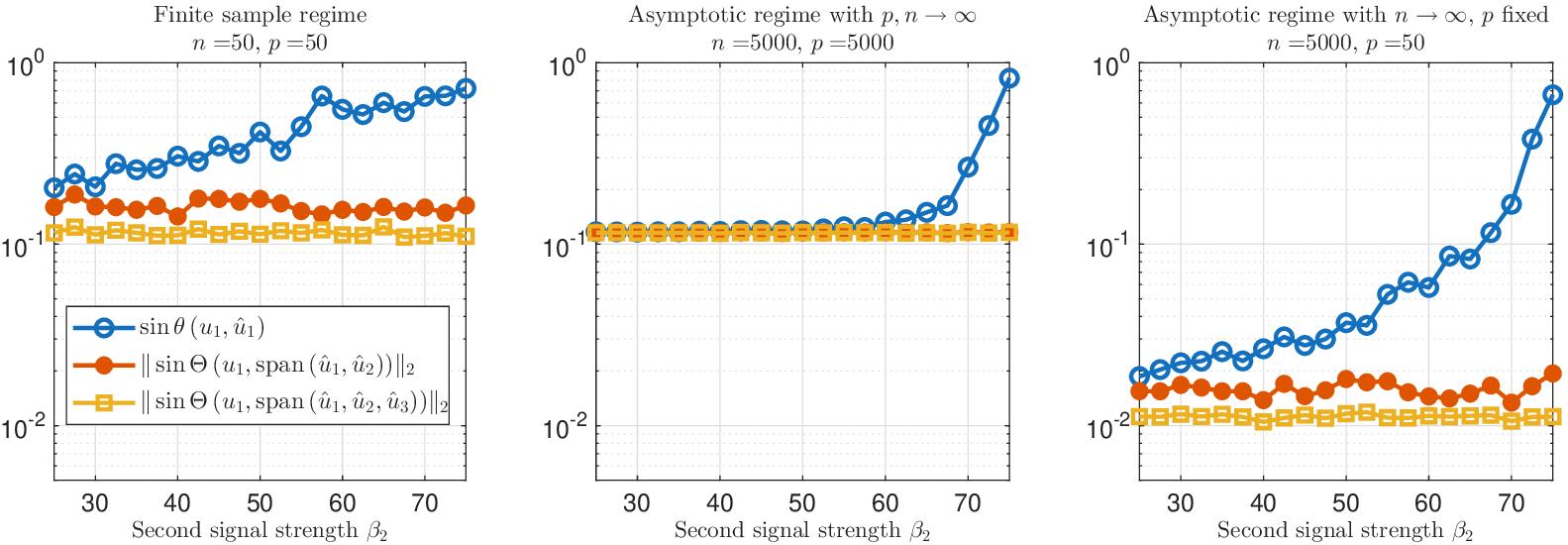}
    \caption{ \small The distance between the dominant signal $u_1$ and spaces spanned by the dominant principal components $\hat{u}_1$, $\hat{u}_2$, and $\hat{u}_3$ for various finite sample and asymptotic regimes. The plotted quantity is the sine of the angle(s) between the dominant population component $u_1$ and a subspace spanned by different sample components $\hat{u}_j$. A small quantity indicates $u_1$ is well-contained in the corresponding subspace. In each of the plots, $\beta_1 =75$ and $\beta_3 = 25$ are kept constant and only $\beta_2$ is changed. The lines show the mean of 20 iterations.} 
    \label{fig:introductionSimple}
\end{figure}

\begin{figure}
    \centering
    \includegraphics[width = 0.7\linewidth]{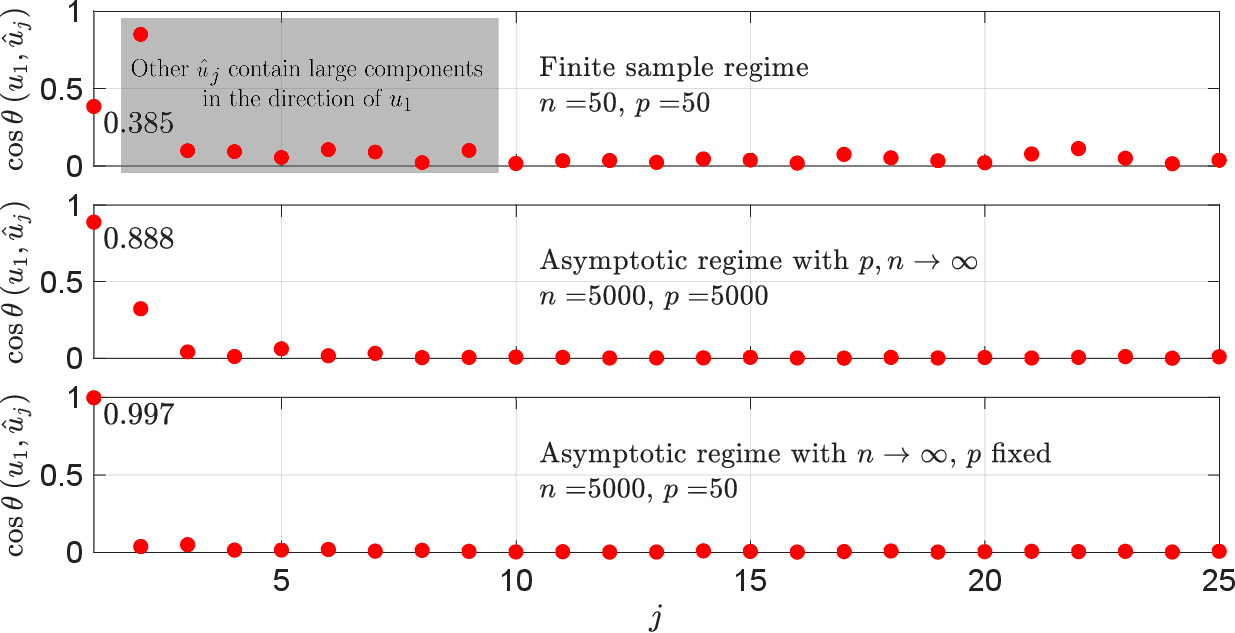}
    \caption{ \small How much do sample components point in the direction of a specific principal component? Here, $\beta_1 = 10$, $\beta_2 = 9, \dots, \beta_9 = 2$, and $\sigma^2 = 1$.}
    \label{fig:introductionCos}
\end{figure}

\subsection{Contributions}
In this paper we investigate the accuracy of the sample principal components in the finite sample setting where the data contains multiple signals ($k>1$). More precisely, we introduce
\begin{itemize}
    \item new bounds on the angle between the population principal components and a subspace spanned by sample principal components, where the latter can be of larger dimension than the population subspace that is estimated (Section \ref{sec:theory}), and
    \item a new algorithm SuperPCA, which  finds an improved estimate of the leading signal directions in a (potentially larger) candidate subspace $\image(\hat U_\mathrm{cand})$. SuperPCA hinges on the idea of the classical Rayleigh-Ritz (RR) process \cite[Sec. 11.3]{parlettsym}, but 
    crucially, \emph{subsamples} only a carefully chosen set of coordinates of the measurements in the refinement step, which reduces the cost of data acquisition, potentially dramatically (Section \ref{sec:superPCA}).
\end{itemize}
As our numerical experiments (Section \ref{sec:experiments}) show, SuperPCA is particularly effective for highly coherent signal directions and for certain distributions of the signal. We compare our algorithm to standard PCA and to Johnstone and Lu's SparsePCA algorithm \cite{Johnstone2009}.

\subsection{Notation}
Throughout the paper, we assume model \eqref{eq:1.defmodel}. The signal directions are denoted as $u_1,\dots,u_k$ and collected in a matrix $U\in\Re^{p\times k}$. We let $X\in\Re^{p\times n}$ denote the data matrix whose columns are samples, and $X_n= X/\sqrt{n}$ as the scaled data matrix. The singular value decomposition of $X_n = \hat{U}\hat{\Sigma}\hat{V}^T$ is denoted with hats. We denote $\sigma_i$ the singular values of $K_{\mathrm{pop}}$ and $\sigma_i(M)$ the singular values of any matrix $M$. The notation $\Theta(V,W)$, for $V\in\Re^{m\times n}$, $W\in\Re^{m\times l}$, and $m\geq \max(n,l)$, denotes the diagonal matrix of dimension $\min(n,l)$ of canonical angles between the subspaces spanned by $V$ and $W$. 
The norms $\|\cdot\|_2$ and $\|\cdot\|_F$ denote the spectral and Frobenius norm, respectively. We use MATLAB notation to indicate submatrices, e.g. $A(i,:)$ is the $i$th row of $A$.

%% file: 2.LitReview_new.tex
\section{Review of related work}\label{sec:review}
In this section we summarize previous work on the estimation of principal components. 
The existing results are mostly stated in terms of the eigendecompositions; we translate them to an SVD context. Finally we also discuss classical perturbation results for singular vectors.

\subsection{Estimation in the limit \boldmath{$p/n\to\gamma > 0$}}\label{sec:LitReviewAsymp}
The high-dimensional asymptotic regime, in which $p,n\to\infty$ with $p/n\to\gamma>0,$ has been extensively studied using random matrix theory. A cornerstone result is the Marchenko-Pastur law~\cite{Marcenko1967b}, which describes the limiting spectrum of the sample covariance matrix in the absence of signals. The convergence of the largest and smallest sample eigenvalues to the upper~\cite{geman1980limit,johnstone2001distribution} and lower~\cite{bai1999methodologies,bai1993limit,silverstein1985smallest} edges were previously studied, showing that the sample eigenvalues are more dispersed than the population eigenvalues~\cite{johnstone2018pca}. 
For the spiked covariance model~\eqref{eq:1.defmodel} introduced by Johnstone~\cite{johnstone2001distribution}, Baik, Ben Arous and Péché~\cite{baik2005phase} and Baik and Silverstein~\cite{Baik2006} established the BBP phase transition: only spikes exceeding a critical threshold, namely $\sigma_i^2/\sigma^2 > 1 + \sqrt{\gamma}$, separate from the bulk spectrum (predicted by the Marchenko-Pastur law) and can be consistently detected.

A similar phase transition holds for the sample principal components. Paul \cite{Paul2007} and subsequent work \cite{yao2015sample} showed that, above the BBP threshold, the sample eigenvectors retain nontrivial alignment with the population eigenvectors but remain asymptotically biased whenever $\gamma>0$. Below the threshold (e.g. when $\sigma_1^2/\sigma^2 \leq 1 + \sqrt{\gamma}$), the sample and population eigenvectors become asymptotically orthogonal. These results motivate the search for improved estimators in the finite-sample regime considered in this work.

\subsection{Finite sample estimation}
The context where $p$ and $n$ are both 
finite and of the same order is of growing relevance and attention in light of the rapid growth of the size of datasets available today, 
however theoretical results are still significantly less developed for this context. Paul's work~\cite{Paul2007} proves that, for finite $p,n$, sample principal components will contain a non-informative noise part that obstruct good estimation of the signals as in the asymptotic case.

Nadler (2008)~\cite{Nadler2008} was among the first to investigate the finite sample case in more detail. He specifically considered the single signal case ($k=1$) using eigenvalue and -vector perturbation theory and derives upper and lower bounds on the relevant quantities in terms of random variables. 
A simplified version of the main result --- neglecting small terms and assuming the signal is sufficiently strong --- is
$\sin\theta(u_1,\hat{u}_1) \lesssim \frac{\sigma}{\sqrt{\beta_1}}\sqrt{\frac{p}{n}} + \mathcal{O}(\sigma^2).$

There have been a few other results in the finite sample setting in the last decade. Koltchinskii and Lounici~\cite{koltchinskii2017normal} and Reiss and Wahl~\cite{reiss2020nonasymptotic} 
derive error bounds on $\sin \theta (u_j,\hat{u}_j)$ and the general $k>1$ case $\sin \Theta (U_1,\hat{U}_1)$, where $U_1$ and $\Uh_1$ contain $r\leq k$ principal components respectively, which depend on the trace of $K_{\mathrm{pop}}$ and involve a spectral gap of the form $\sigma_j^2-\sigma_{j+1}^2$ (respectively $\sigma_{r}^2-\sigma_{r+1}^2$) in their denominator. 
To be informative, their bounds require roughly that $n\gg p\sigma^2$~\cite{koltchinskii2017normal} or $n\sim p^2$~\cite{reiss2020nonasymptotic}, and that the spectral gap $\sigma_j^2-\sigma_{j+1}^2$ is sufficiently large. 
Vaswani and Narayanamurthy~\cite{vaswani2017finite} obtain error bounds on $\|\sin\Theta(U_1,\hat{U}_1)\|_2$ that behave as $\sim\sqrt{(\beta_1/\beta_r)(\sigma^2/\beta_r)(p\log p /n)}$, which again includes spectral gaps as $\sigma^2/\beta_r$ and a $p/n$ term.
In all these references, the dimensions of $U_1$ is the same as that of $\hat U_1$. Hence a common feature of these results is that the error bounds depend on spectral gaps between adjacent singular values, making them less informative when several leading signals have similar strengths. Our analysis addresses this limitation by estimating signals using a subspace of larger dimension than the target population subspace.

\subsection{Matrix perturbation theory for singular vectors}\label{subsec:wedin}
Our analysis of the error in the sample principal components hinges on classical perturbation theory for invariant subspaces. Indeed we can view the sample matrix $S_n$ as a perturbed version of $K_{\mathrm{pop}}$ or, in the SVD formulation, the data matrix $X_n$ is a noisy version of the signal part of ~\eqref{eq:1.defmodel}. For this task, Davis and Kahan \cite{dk70} established perturbation bounds for eigenspaces of symmetric matrices, while Wedin \cite{Wedin1972} obtained analogous results for singular vectors.

A key feature of Wedin's theorem is that the perturbation bound depends on a spectral gap separating the desired singular subspace from the remaining spectrum. However, when several signal strengths are close together such that this gap becomes small, Wedin's theorem becomes less informative. The main idea of this paper is to enlarge the approximation subspace, replacing the adjacent spectral gap by a larger separation that yields substantially sharper finite-sample bounds.

%% file: TheoryAngles_new.tex
\section{Main theoretical results: estimation of principal components using a larger subspace} \label{sec:theory}
We now quantify the accuracy with which the leading population signal subspace can be approximated by a larger sample singular subspace.
Throughout, let $U_1\in \R^{p \times r_1}$ denote the leading population signal directions and let $[\Uh_1, \Uh_2]\in \R^{p\times r_2}$, with $\Uh_1\in \R^{p\times r_1}$ and $r_2>r_1$, denote the leading left singular vectors of the data matrix $X_n$. The approximation error is measured using canonical angles between these subspaces, which measure how much of the smaller subspace is contained in the larger one. We use the standard definition for subspaces of unequal dimensions \cite[Section I.5]{MPT1990}: 
the canonical angles between 
$U_1\in \R^{p\times r_1}$ and $[\Uh_1,\Uh_2]\in \R^{n\times r_2}$ are defined as
\begin{equation}
    \Theta(U_1,[\Uh_1,\Uh_2]) = \mathrm{diag}(\theta_1,\theta_2,...,\theta_{r_1}), \text{ where }\theta_i = \arccos (\sigma_i (U_1^T[\Uh_1,\Uh_2])).
\end{equation}
We further denote $\Uh_3$ an orthogonal complement of $[\Uh_1,\Uh_2]$; that is, $\Uh_3^T[\Uh_1,\Uh_2] = 0$.
Noting that $\sin^2\theta+\cos^2\theta = 1$ and that $\left\|U_1^T[\Uh_1,\Uh_2]\right\|_2^2 +\left\|U_1^T\Uh_3\right\|_2^2=1$, a direct consequence of the definition above is that, 
\begin{equation}
    \left\|\sin \Theta (U_1,[\Uh_1,\Uh_2])\right\|_2 = \left\|U_1^T\Uh_3\right\|_2,
\end{equation}
which is the quantity that we will use to measure the error made when we approximate the principal components $U_1$ by $\image([\Uh_1,\Uh_2])$.

We are now ready to state our theoretical results. 
First, we establish an upper bound for the case when there is only one signal of interest ($r_1=1$) by considering a subspace spanned by the leading left singular vectors of the data matrix; see Theorem \ref{thm:onesigbound}. Then, we extend our result to the case of multiple signals in Theorem \ref{thm:multsigbound}.

\subsection{Estimating one signal}
We derive a deterministic bound for the canonical angle between the dominant signal and the subspace spanned by the leading left singular vectors of the data matrix $X_n$. The proof of the theorem below can be found in Section \ref{sec:proofs}.

\begin{theorem}[Single signal case] \label{thm:onesigbound}
    Let $X_n = \frac{1}{\sqrt{n}}X=\frac{1}{\sqrt{n}}[x_1,x_2,...,x_n] \in \Re^{p\times n}$ be the scaled data matrix where the data matrix $X$ is as defined in \eqref{eq:1.defmodel}. Then assuming $\sigma_1(X_n) > \sigma_{r_1+1}(X_n)$,
    \begin{equation} \label{eq:onesigbound}
        \sin\theta(u_1,\Uh_1) \leq \sqrt{\frac{\sigma_1(X_n)^2-\norm{u_1^TX_n}_2^2}{\sigma_1(X_n)^2-\sigma_{r_1+1}(X_n)^2}}
    \end{equation} where $u_1$ is the dominant signal direction and $\Uh_1 \in \Re^{p\times r_1}$ is the matrix containing the $r_1$ leading left singular vectors of $X_n$.
\end{theorem}

The bound for $r_1>1$ is useful in the multiple signals case $(k>1)$ when the dominant signal $u_1$ is (almost) indistinguishable from the non-dominant signal(s), that is, the largest signal strength and the second largest signal strength are similar. By taking a larger subspace of dimension $r_1>1$, $\sigma_1(X_n)^2-\sigma_{r_1+1}(X_n)^2$ can be substantially larger than $\sigma_1(X_n)^2-\sigma_{2}(X_n)^2$, making the bound \eqref{eq:onesigbound} more informative.
Furthermore,
\begin{equation*}
    \sigma_1(X_n) = \norm{\tilde{u}_1^T X_n}_2 \approx \norm{u_1^TX_n}_2,
\end{equation*} where $\tilde{u}_1$ is the leading singular vector of the data matrix, which implies a small numerator.

\subsection{Estimating multiple signals}
We now extend Theorem \ref{thm:onesigbound} to estimating multiple population principal components. In particular, we consider the quantity $\norm{\sin\Theta(U_1,[\Uh_1,\Uh_2])}_2$ where $U_1,\Uh_1 \in \R^{p\times r_1}$ and $\Uh_2\in \R^{p\times (r_2-r_1)}$. The orthonormal matrix $U_1$ is the dominant $r_1$ signal directions and $[\Uh_1,\Uh_2]$ is the leading $r_2$ left singular vectors of the data matrix $X_n$. The result is presented below in Theorem \ref{thm:multsigbound} with its proof in Section \ref{sec:proofs}.

\begin{theorem} \label{thm:multsigbound}
    Let $X_n =\frac{1}{\sqrt{n}}X = \frac{1}{\sqrt{n}} [x_1,x_2,...,x_n]\in \R^{p\times n}$ be the scaled data matrix where the data matrix $X$ is as defined in \eqref{eq:1.defmodel}. Let $U_1\in\mathbb{R}^{p\times r_1}$ be the $r_1$ leading signal directions with $U_\perp\in\mathbb{R}^{p\times (p-r_1)}$ as its orthogonal complement, and let $[\Uh_1,\Uh_2]$ be the leading $r_2(\geq r_1)$ left singular vectors of the data matrix. 
    Then assuming $\sigma_{\min}(U_1^TX_n) > \sigma_{r_2+1}(X_n)$,
    \begin{equation} \label{eq:multisigbound}
        \norm{\sin\Theta(U_1,[\Uh_1,\Uh_2])}_2 \leq \frac{\norm{U_\perp^T X_n \Vc}_2 \sigma_{\min}(U_1^TX_n)}{
        \sigma_{\min}(U_1^TX_n)^2 - \sigma_{r_2+1}(X_n)^2},
    \end{equation} where $\Vc$ is the leading $r_1$ right singular vectors of $U_1^TX_n$, which is independent of $U_\perp^TX_n$.
\end{theorem}

The bound above is useful when the largest $r_1$ signal strength is similar in magnitude to the $(r_1+1)$th strongest signal strength. Then by taking a larger subspace of dimension $r_2$ such that $\sigma_{r_2+1} = \sqrt{\beta_{r_2+1}+\sigma^2}\approx \sigma_{r_2+1}(X_n)$ is sufficiently smaller than $\sigma_{r_1} = \sqrt{\beta_{r_1}+\sigma^2}\approx \sigma_{\min}(U_1^TX_n)$, the estimate $[\Uh_1,\Uh_2]$ is likely to contain important information about the dominant signal directions $U_1$.
\begin{remark}\*
    Since $\Vc$ is independent of $U_\perp^TX_n$, 
    \begin{equation*}
        U_\perp^T X_n \Vc \overset{d}{=}  \diag(\sigma_{r_1+1},...,\sigma_p)G_{(p-r_1)\times r_1}
    \end{equation*}
    where $\sigma_j = \sqrt{\beta_j+\sigma^2}$ and $G_{(p-r_1)\times r_1}$ is a $(p-r_1)\times r_1$ standard Gaussian matrix. We also have
        $U_1^T X_n \overset{d}{=} \diag(\sigma_1,...,\sigma_{r_1}) G_{r_1 \times n}$ and
        $U_\perp^T X_n \overset{d}{=} \diag(\sigma_{r_1+1},...,\sigma_{p}) G_{(p-r_1) \times n}$.
    Therefore, roughly,
    \begin{equation} \label{eq:roughLimit}
        \frac{\norm{U_\perp^T X_n \Vc}_2 \sigma_{\min}(U_1^TX_n)}{
        \sigma_{\min}(U_1^TX_n)^2 - \sigma_{r_2+1}(X_n)^2} \approx \frac{\sigma_{r_1}\sigma_{r_1+1}(\sqrt{p-r_1}+\sqrt{r_1})(\sqrt{n}+\sqrt{r_1})}{(\sqrt{n}+\sqrt{r_1})^2\sigma_{r_1}^2- (\sqrt{n}+\sqrt{p})^2\sigma_{r_2+1}^2},
    \end{equation} which can be made more precise by a result by Davidson and Szarek \cite{DavidsonSzarek2001}, given that the denominator is positive. When $n\rightarrow \infty$ with $p$ fixed, the right-hand side of \eqref{eq:roughLimit} goes to zero and when $p,n \rightarrow \infty$ with $p/n \rightarrow c$, the right-hand side of \eqref{eq:roughLimit} tends to 
    \begin{equation*}
        \frac{\sigma_{r_1}\sigma_{r_1+1}\sqrt{c}}{\sigma_{r_1}^2- (1+\sqrt{c})^2\sigma_{r_2+1}^2}
    \end{equation*} assuming that the denominator remains positive, that is, $\sigma_{r_1}> (1+\sqrt{c})\sigma_{r_2+1}$.

\end{remark}

To illustrate the accuracy of our bounds, in Figure ~\ref{fig:bounds_u1} we plot the error and the bounds from Theorems \ref{thm:onesigbound} and \ref{thm:multsigbound} over the data sample size $n$. The experiment shows that both theorems explain well the trend of the error. We note that the bounds are only informative when smaller than 1.

\begin{figure}[ht]
    \centering
    \includegraphics[width=0.8\linewidth]{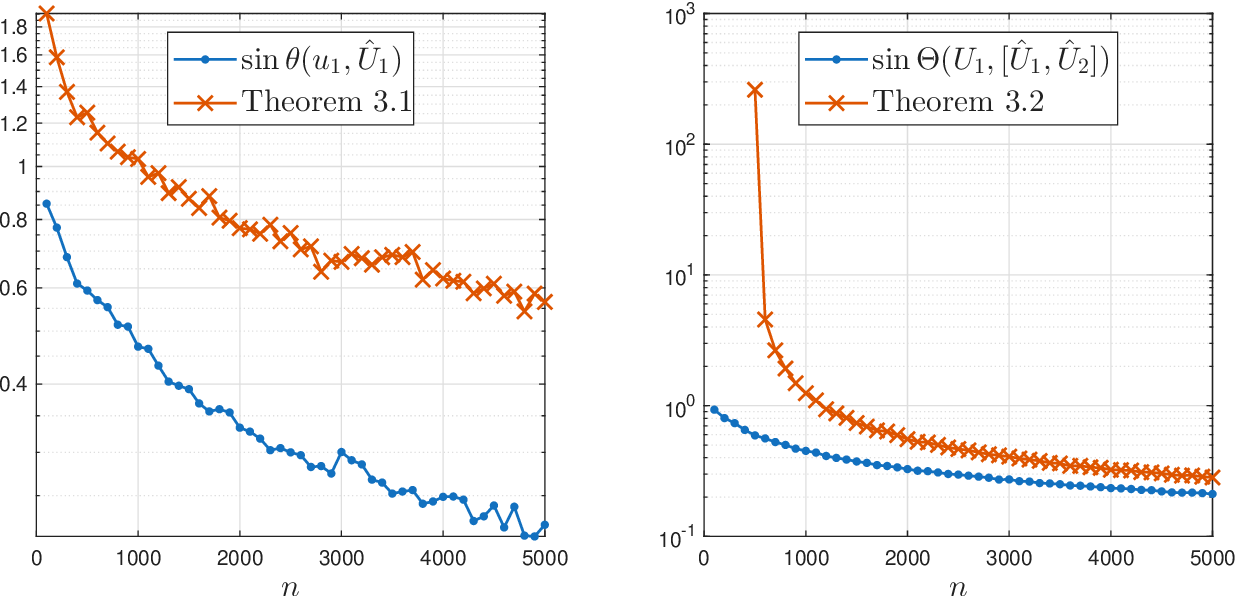}
    \caption{
    Error and bounds from Theorems \ref{thm:onesigbound} and \ref{thm:multsigbound}. Here the data is of dimension $p=1000$, we take $k=10$ signals of strengths $\beta_i = (1000-i^3)/200+1$, and set $\mathrm{dim}\,(\hat U_1)=5, \mathrm{dim}\,([\hat U_1,\hat U_2])=10$. The noise level is $\sigma=1$. The bounds and the error were averaged over 10 runs. }
    \label{fig:bounds_u1}
\end{figure}

%% file: SuperPCA_new.tex
\section{SuperPCA: subsampling to extract high-accuracy principal components from a subspace} \label{sec:superPCA}
The results of Section \ref{sec:theory} suggest that the leading signal can often be recovered more accurately from a larger candidate subspace than from the leading sample principal component alone. In this section, we propose an algorithm to refine a large singular
subspace down to the desired dimension using very limited additional data acquisition. In this way, we exploit the information in the large subspace twofold: firstly, by searching only within that subspace for good principal components and secondly, by using it to decide what coordinates of the measurements are most rich in information.

Let $\Uh_{\mathrm{cand}}\in\mathbb{R}^{p\times r_2}$ be a candidate subspace. Our goal is to compute $r_1$ vectors $\tilde u_1,\tilde u_2, \ldots,\tilde u_{r_1}\in \image(\Uh_{\mathrm{cand}})$ that approximate the leading population principal components $U_1\in \R^{p\times r_1}$ more accurately than the initial guess $\Uh_{\mathrm{cand}}(:,1:r_1)$ while acquiring only a small number of additional measurements. 

The candidate subspace $\Uh_{\mathrm{cand}}$ may be obtained from an initial PCA computation (since, as seen in the previous sections, this gives a good candidate) or from prior information. For example, $\Uh_{\mathrm{cand}}$ may come from prior knowledge, such as information that has been acquired from related systems or previous experiments. Throughout this section we only assume that it contains useful information about the desired signal. The quality may be high or low; that is, $\|\sin\Theta(U_1,\Uh_{\mathrm{cand}})\|$ may range from very small, say $10^{-10}$, to relatively large, for example $0.1$, the latter being perhaps the more realistic situation where the signal-to-noise ratio is low.

It is worth mentioning that the most basic approach to improve the estimate for the principal component would be to take  additional samples $X_2$ and take the SVD of $[X_1, X_2]$. 
Clearly this is the classical PCA strategy. Note that this approach does not output a solution in the span of $\Uh_{\mathrm{cand}}$, and therefore its output can be better than $\|\sin\Theta(U_1,\Uh_{\mathrm{cand}})\|$. We shall see however that in many cases SuperPCA outperforms PCA in terms of the number of measurements needed, especially when $p\gg r_2$.

\subsection{Rayleigh-Ritz PCA}
Assume we take $N$ additional measurements from model \eqref{eq:1.defmodel} and collect them in a data matrix $X_2$. 
To estimate the leading principal components, 
we propose to project the newly acquired data $X_2$ onto the candidate subspace and perform PCA in that subspace, namely it suffices to perform PCA on $\Uh_{\mathrm{cand}}^TX_2$. 
This leads us to introducing Algorithm \ref{alg:RRPCA}, which we call Rayleigh-Ritz PCA (RR-PCA) due to its link to the Rayleigh-Ritz process~\cite[Ch.~11]{parlettsym} as we explain below. To the best of our knowledge, RR-PCA is novel and may merit further investigation in its own right.
\begin{algorithm}
  \caption{\small RR-PCA.
Inputs: initial candidate subspace $\Uh_{\mathrm{cand}}$ (obtained from $n$ samples or prior knowledge of the model), additional sample size $N$, and number of desired principal components $r_1$}.
  \label{alg:RRPCA}
  \begin{algorithmic}[1]
  \small
\STATE Take $N$ additional measurements in \emph{all} coordinates $X_2\in\mathbb{R}^{p\times N}$. 
\STATE Compute $\Uh_{\mathrm{cand}}^TX_2$ and its SVD $\Uh_{\mathrm{cand}}^TX_2=U_N\Sigma_NV_N^T$. 
\STATE Output the $r_1$ leading vector(s) of $\Uh_{\mathrm{cand}}U_N$. 
  \end{algorithmic}
\end{algorithm}

Applying RR-PCA to the data $X_2$ corresponds mathematically to finding the leading left singular vectors of $\Uh_{\mathrm{cand}}\Uh_{\mathrm{cand}}^TX_2$, which is an orthogonal projection of $X_2$ onto $\Uh_{\mathrm{cand}}$. Projection is the key idea of subspace methods for eigenvalue problems~\cite{saadbookeig}, and related randomized algorithms such as the popular Randomized SVD~\cite{halko2011finding}. 
Equivalently, RR-PCA finds the (leading) eigenvalues and eigenvectors of the matrix $(\Uh_{\mathrm{cand}}\Uh_{\mathrm{cand}}^T X_2)(\Uh_{\mathrm{cand}}\Uh_{\mathrm{cand}}^T X_2)^T= \Uh_{\mathrm{cand}}\Uh_{\mathrm{cand}}^T(X_2 X_2^T)\Uh_{\mathrm{cand}}\Uh_{\mathrm{cand}}^T$, which is the projection of the covariance matrix $X_2X_2^T$ onto $\Uh_{\mathrm{cand}}$. Therefore RR-PCA applies the classical Rayleigh-Ritz (RR) process~\cite[Ch.~11]{parlettsym} to the covariance matrix $X_2X_2^T$. RR is known to be the optimal strategy for extracting approximate eigenvectors from a candidate subspace in a number of senses~\cite[Ch.~11]{parlettsym}. As $N\rightarrow \infty$, the leading eigenvectors of the sample covariance matrix $X_2X_2^T$ converge to the population principal components, and RR-PCA converges to the best approximation of the population principal components contained in $\image(\Uh_{\mathrm{cand}})$.

From a practical viewpoint, however, RR-PCA is not very attractive because it still requires us to sample the whole dataset $X_2$ and therefore offers no reduction in measurement cost compared to classical PCA. \footnote{An exception to this argument would be a situation where 
one is able to make further measurements in the direction of $\Uh_{\mathrm{cand}}^T$ at the cost proportional to $r_2$ rather than $p$. This seems restrictive, and we shall not make such assumptions in what follows.} 
That is why we next introduce an alternative strategy that exploits coordinate sampling.

\subsection{SuperPCA}
Note that we can reinterpret RR-PCA as a least-squares problem. When we perform the orthogonal projection $\Uh_{\mathrm{cand}}\Uh_{\mathrm{cand}}^TX_2$ as described above, we are effectively solving the least-squares problem (with many right-hand sides)
 \begin{equation}   \label{eq:U_1M}
\min_{M\in\mathbb{R}^{r_2\times N}}\|\Uh_{\mathrm{cand}}M-X_2\|_F. 
 \end{equation}
The solution to this problem is $M_*=(\Uh_{\mathrm{cand}})^\dagger X_2 = \Uh_{\mathrm{cand}}^TX_2$, where $(\cdot)^\dagger$ denotes the Moore-Penrose pseudoinverse. Then $X_2$ can be approximated by $\Uh_\mathrm{cand}M_*$, and by taking the dominant subspace of $\Uh_{\mathrm{cand}}M_*$ we recover the RR solution.

Note that~\eqref{eq:U_1M} is a $p\times r_2$, highly overdetermined least-squares problem ($p \gg r_2$), therefore instead of solving~\eqref{eq:U_1M} exactly, we can approximate its solution with the subsampled least squares problem 
 \begin{equation}   \label{eq:sketdchU_1M}
\min_{\tilde M\in\mathbb{R}^{r_2\times N}}\|S(\Uh_{\mathrm{cand}}\tilde M-X_2)\|_F, 
 \end{equation}
where $S\in\mathbb{R}^{s\times p}$ is a \emph{subsampling} matrix, i.e. whose rows have only one entry equal to $1$, and $0$ elsewhere. Let $\mathcal{I}\subseteq \{1,2,\ldots,p\}$ denote the indices chosen by $S$. We require $|\mathcal{I}|=:s>r_2$, and usually $s\ll p$. 
This yields the SuperPCA algorithm, presented in Algorithm~\ref{alg:superPCA}.

\begin{algorithm}[htb]
  \caption{\small SuperPCA. 
Inputs: initial sample size $n$, additional sample size $N$, and subspace dimension $r_2$ or a threshold $\tau\in (0,1)$. Optionally: approximate subspace $\Uh_{\mathrm{cand}}\in\mathbb{R}^{p\times r_2}$ (if $\Uh_{\mathrm{cand}}$ is given, start from step 3).
}
  \label{alg:superPCA}
  \begin{algorithmic}[1]
  \small
\STATE Take $n$ measurements to obtain $X_1\in\mathbb{R}^{p\times n}$. 
\STATE Compute the SVD $X_1=\Uh\hat\Sigma \hat V^T$, and let 
$\Uh_{\mathrm{cand}}=\Uh(:,1:r_2)$. If $r_2$ is not given, let $r_2$ be smallest such that $\sigma_{r_2+1}(X)/\sigma_{1}(X)\leq \tau$.
\STATE Find ``important'' row indices 
$\mathcal{I}\subseteq \{1,2,\ldots,p\}$ of $\Uh_{\mathrm{cand}}$ 
via repeated and reweighted QRCP; see Section \ref{sec:subsamplingMethod}. 
\STATE Take $N$ more measurements $\tilde X_2\in\mathbb{R}^{s \times N}$, only in the indices $\mathcal{I}$.
\STATE Solve the least-squares problem with Tikhonov regularization $\min_M\|\Uh_{\mathrm{cand}}(\mathcal{I},:)M-\tilde{X}_2\|^2_F + \lambda\|M\|^2_F$ via the QR factorization of $[\Uh_{\mathrm{cand}}(\mathcal{I},:); \sqrt{\lambda}I_r]$ (see Section \ref{subsec:reg}).
\STATE Find the economical SVD $M=U_M\Sigma_MV_M^T$. 
\STATE Output the $r_2$ leading vectors of $\Uh_{\mathrm{cand}}U_M\in\mathbb{R}^{p \times r_2}$.
  \end{algorithmic}
\end{algorithm}

Solving \eqref{eq:sketdchU_1M} has two important advantages. First, it is much faster than solving \eqref{eq:U_1M}, while giving solutions that are almost as good as \eqref{eq:U_1M}. In fact, there is now a rich body of theory that justifies solving~\eqref{eq:sketdchU_1M} often yields solutions of comparable quality to \eqref{eq:U_1M}. For example, with leverage-score sampling (where $S$ includes an importance-sampling weighting) one has $\|\Uh_{\mathrm{cand}}\tilde M_*-X_2\|_F\leq (1+\epsilon)\|\Uh_{\mathrm{cand}}M_*-X_2\|_F$, where $\epsilon$ is the so-called subspace embedding constant~\cite{drineas2012fast}, for which a typical value is say $1/2$. Thus if $\Uh_{\mathrm{cand}}$ captures the signal well such that $\|\Uh_{\mathrm{cand}}M_*-X_2\|_F$ is small, the subsampled solution has a good fit with small $\|\Uh_{\mathrm{cand}}\tilde M_*-X_2\|_F$. 
Solving heavily overdetermined least-squares problems via subsampling is a fundamental idea that has been successfully used in e.g. model order reduction~\cite{chaturantabut2010nonlinear} and numerical linear algebra~\cite{jones2025subapsnap}.

The second advantage of SuperPCA is much more fundamental. To get the solution of \eqref{eq:sketdchU_1M}, SuperPCA only needs new measurements on the sampled coordinates $\mathcal{I}$. Consequently, the measurement cost is reduced by approximately a factor $p/s \gg 1$. 
Indeed, we are assuming that taking measurements at a subset $|\mathcal{I}|=s$ of indices can be done with cost proportional to $s$. 
Such situation is common in e.g. model order reduction~\cite{chaturantabut2010nonlinear}, nonetheless
it is important to acknowledge that this is an assumption rather than a fact, and in some applications it may be equally expensive to obtain $S X_2$ as it is to obtain $X_2$.

The remainder of this section discusses two practical ingredients of SuperPCA: the choice of sampled coordinates $\mathcal{I}$, 
and the regularization used in solving the sketched least-squares problem.

\subsection{Selecting subsampling indices} \label{sec:subsamplingMethod}
\subsubsection{QRCP with iterative reweighting}
A number of algorithms are available for choosing a subset of columns (or rows) that are ``important'' from a given matrix. 
These include classical algorithms in numerical linear algebra such as QR with column pivoting (QRCP) or Gaussian elimination with row pivoting~\cite[Ch.~3,5]{golubbook4th}, \cite{dong2023simpler}. Simpler methods  include uniform random sampling or sampling based on the column norms; however they are known to fail for difficult problems. Another option is the use of leverage score sampling~\cite{mahoney2011randomized} for which extensive research has been performed. Other choices include the randomly pivoted Cholesky algorithm~\cite{chen2022randomly} and determinantal point process~\cite{kulesza2012determinantal}. 
 
The method we advocate here is motivated by the observation that QRCP  combines excellent performance and speed in practice~\cite{dong2023simpler,drmac2016new}, despite the fact that in the worst case it can be exponentially far from the optimal choice of indices. 
We adopt an iterative and reweighting process introduced in~\cite{park2025accuracy}, which selects the rows as follows: first perform QRCP as usual to obtain the $r_2$ indices $\mathcal{I}$, then compute the SVD of the submatrix $\Uh_{\mathrm{cand}}(\mathcal{I},:)=U_{\mathcal{I}}S_{\mathcal{I}}V_{\mathcal{I}}^T$.  
We then compute the matrix 
$\Uh_{\mathrm{cand}}^{(2)} := \Uh_{\mathrm{cand}}(\mathcal{I}^C,:)V_{\mathcal{I}}S_{\mathcal{I}}^{-1}$, where $\mathcal{I}^C$ denotes the complement indices of $\mathcal{I}$ in $\{1,2,\ldots, p\}$. 
We then run QRCP on $\Uh_{\mathrm{cand}}^{(2)}$ to get $r_2$ additional indices. This process can be repeated. 
The idea is that we wish to emphasize the directions that we have failed to capture so far. The $V_{\mathcal{I}}$-multiplication gets the matrix in the right coordinates, and the further $S_{\mathcal{I}}^{-1}$-multiplication promotes the directions not yet captured, as described in~\cite{park2025accuracy}.

As shown in~\cite{park2025accuracy}, this method leads to a relatively large value of $\sigma_{\min}(\Uh_1(\mathcal{I},:))$, which is the quantity that controls the suboptimality factor of the least-squares fit~\cite{drmac2016new,jones2025subapsnap}. 
We observe that this method performs well, in particular often better than leverage score sampling in that it results in a larger value of $\sigma_{\min}(\Uh_1(\mathcal{I},:))$  for the same cardinality $|\mathcal{I}|$. Comparing the accuracy achieved by SuperPCA with leverage score sampling and with the sampling method described above we observed that the above method is more robust. With leverage scores sampling, if the subsampled least squares problem is weighted by the inverse square roots of the leverage scores as is usually done, then SuperPCA often did not improve the estimate of the leading signal, however if we do not use weights then the accuracy is comparable to the subsampling method above with the condition that the candidate subspace $\Uh_1$ is accurate enough.

\subsubsection{Coherence of $\Uh_{\mathrm{cand}}$ and number of selected coordinates}
The coherence of a matrix with orthonormal columns measures how strongly its column space is aligned with the coordinate axes~\cite{ipsen2014effect}. In particular, for $\Uh_{\mathrm{cand}}\in\mathbb{R}^{p\times r_2}$, we define the coherence as 
\[
\mu(\Uh_{\mathrm{cand}})
=
\frac{p}{r_2}
\max_{i=1,\ldots,p}
\|\Uh_{\mathrm{cand}}(i,:)\|_2^2.
\]
The coherence satisfies
$1 \leq \mu(\Uh_{\mathrm{cand}}) \leq \frac{p}{r_2}$.
When $\mu(\Uh_{\mathrm{cand}})$ is close to $1$, the matrix is said to be \textit{incoherent}, meaning that its row norms are relatively uniform and its column space is not strongly concentrated along any individual coordinate direction. In contrast, high coherence indicates that the column space is strongly aligned with a small number of coordinate directions.
Coherence therefore plays an important role in determining the accuracy of a subsampled least-squares solution as an approximation to the full least-squares solution~\cite{drineas2006sampling}. In SuperPCA, the coherence of $\Uh_{\mathrm{cand}}$ consequently influences the number of rows $s$ that must be selected to obtain an accurate approximation, as we demonstrate in Section~\ref{sec:experiments}.

\subsection{Regularization}\label{subsec:reg}
\subsubsection{Tikhonov regularization}
We suggest to include Tikhonov regularization in solving the least squares problem ~\eqref{eq:sketdchU_1M}. 
The regularized solution is then obtained by solving the least squares problem 
\begin{equation}\label{eq:LSreg}
    \min_{M\in\mathbb{R}^{r_2\times N}} \|(S \Uh_{\mathrm{cand}})M-SX_2\|_F^2  + \lambda \|M\|_F^2.
\end{equation}
We found experimentally that regularization substantially improves the robustness of SuperPCA when the signal-to-noise ratio is moderate or when the candidate subspace dimension is chosen larger than the number of dominant signals. 
The regularization parameter $\lambda$ is selected automatically using the classical L-curve criterion \cite[sec. 3.6.4]{bjorckBook}\cite{hansenOLeary}, which balances the residual norm against the solution norm. In practice, we maximize the curvature of the L-curve (whose analytic expression can be computed as shown in~\cite{hansen1999curve}) using MATLAB's function ${\tt fminbnd}$. Other automatic regression parameter rules (for example GCV) are possible.

\subsubsection{Regularization and choice of the dimension of the candidate subspace} 
In real applications, the number of signals $k$ is unknown, therefore to define the size $r_2$ of the candidate subspace in SuperPCA the number of signals $k$ needs to be estimated. Classical approaches to estimate the number of principal components include heuristic methods such as the scree plot \cite{cattell1966scree} and explained-variance thresholds, statistical approaches such as Horn’s parallel analysis \cite{horn1965rationale}, and information-theoretic criteria such as AIC and MDL \cite{wax1985detection, nadakuditi2008sample}. In high-dimensional statistics, the data is often modeled by \eqref{eq:1.defmodel} and several practical methods to estimate the number of components were derived using random matrix theory \cite{nadakuditi2008sample,kritchman2009non}. These methods are fundamentally linked to the BBP phase transition \cite{baik2005phase}.

In our experiments we have set $\mbox{dim}(\Uh_{\mathrm{cand}}) = r_2=k$.
Assuming the data follows the spiked covariance model, if the signal strengths are well above the noise $\beta_i > \sigma^2$ then the singular values of the initial sample $X_1$ that we use to obtain $\Uh_{\mathrm{cand}}$ have an important spectral gap $\sigma_1 - \sigma_{k+1}$. 
Note that using $r_2$ large may be unnecessary if only the first principal component is desired, however it improves the guarantees provided in our theoretical results and experiments confirm that it is not harmful to take $r_2 \approx k$ for example.

On the other hand, if one wishes to approximate many principal components then SuperPCA may benefit from taking $r_2>k$. Indeed, similarly to the situation in Figure \ref{fig:introductionCos}, the sample components $i>k$ may contain information about some desired population principal components. In this situation, regularization plays an important role: without Tikhonov regularization if $r_2>k$, if the signal is weak (e.g. $\beta_1 \approx \sigma^2$) or $n$ is small (e.g. $n=300$ for $p=1000,k=10,\beta_i = (1000-i^3)/200+1$) then SuperPCA gives an inaccurate principal component.
Indeed, when $r_2>k$, we observe that the subsampled $S\Uh_{\mathrm{cand}}$ has additional singular values that are considerably smaller than $1$, so taking the pseudo-inverse of $S\Uh_{\mathrm{cand}}$ may amplify the noise in the data $X_2$. In this situation Tikhonov regularization with the L-curve method for the choice of regularization parameter ensures that SuperPCA is robust and improves the estimation of the principal component.

\subsection{SuperPCA vs. Sparse PCA}
SuperPCA is related to sparse PCA~\cite{zou2006sparse, Johnstone2009, truncatedPower, journee2010generalized,kumar2024oja}, but the underlying assumptions are different. Sparse PCA assumes that the population principal components are sparse in the given coordinate system and seeks to recover their support together with the principal components themselves. By contrast, SuperPCA makes no sparsity assumption on the columns of $U_1$. Instead, it assumes the availability of a candidate subspace containing the desired signals and uses this subspace to select informative coordinates for subsequent measurements.

When the population principal components are sparse or approximately sparse, the candidate subspace is naturally concentrated on a small number of coordinates. In this case, SuperPCA selects essentially the same coordinates as sparse PCA~\cite{Johnstone2009} and therefore benefits from similar reductions in measurement cost. However, SuperPCA remains applicable even when the signals are not sparse, provided the candidate subspace is sufficiently informative.

Unlike iterative sparse PCA methods such as the truncated power method~\cite{truncatedPower}, GPower~\cite{journee2010generalized}, or sparse Oja's algorithm~\cite{kumar2024oja}, SuperPCA selects the measurement coordinates only once and then refines the principal components by solving a sketched least-squares problem. Nevertheless, the power iteration algorithm is tightly connected to RR~\cite[Ch. 11]{parlettsym}, therefore we acknowledge the similarity between these methods and SuperPCA.

%% file: ExperimentsSPCA.tex
\section{Numerical experiments} \label{sec:experiments}
In this section we illustrate the behavior of the SuperPCA algorithm through numerical experiments, first on synthetic data then on a real data set. Using synthetic data, we test the SuperPCA algorithm for different distributions of the signal strength $\beta_i$ and different structures of the population principal components $U_1$. 
Further, we illustrate that splitting optimally the budget in terms of the number of measurements between the initial data $X_1$ and the subsampled data $SX_2$ is not straightforward and depends notably on the distribution of the signal.

\paragraph{Synthetic experiments set-up}
In the experiments with synthetic data, we define the matrix $U_1$ containing the population principal components as the orthogonal factor of the thin QR decomposition of the matrix $[C_{\mathrm{big}}\times G_1; G_2]$, where $G_1 \in \R^{n_{\mathrm{big}}\times k}$ and $G_2\in \R^{(p-n_{\mathrm{big}})\times k}$ are Gaussian matrices, $n_{\mathrm{big}}\leq p$ and $C_{\mathrm{big}}$ is a real number larger than $1$ (whose value is fixed in the experiments below). 
We generate the sample $X_1$ as in the model \eqref{eq:1.defmodel}, with dimension $p=1000$ and $k=10$ signals, and we subsample the follow up data $X_2$ only at $s$ rows. Unless stated otherwise, we set the noise level to $\sigma=1$.


\subsection{Refinement of the leading principal component starting from a candidate subspace \boldmath{$\Uh_1$}} \label{subsec:experiments1}
First, we investigate the effect of additional measurements on the accuracy of the approximate leading principal component computed with SuperPCA.
We consider a fixed candidate subspace $\Uh_1\in \R^{p\times r}$ obtained from applying PCA to a synthetic sample $X_1$ (following \eqref{eq:1.defmodel} with the parameters defined above) of size $n=10p$ and retaining the $r=10$ leading principal components of $X_1$. We look at the error $\sin \theta (u_1,u^*)$ as the number of columns $N$ of $SX_2$ increases. We denote $\hat u_1=\Uh_1(:,1)$ the initial guess. In the figures, we compare the error from SuperPCA with that of RR-PCA and classical PCA for the same number of measurements. Therefore, since classical PCA and RR-PCA use the entire data vectors, they are applied to data matrices that have $sN/p$ columns instead of $N$ columns, which is indicated in the legends by ``small sample''. The data $X_1$ used to compute $\Uh_1$ is also included in the subsampled data $SX_2$.

In practice the subspace $\Uh_1$ could be obtained from prior knowledge of the system or a previous computation of the principal components. For example, in the context of model order reduction~\cite{chaturantabut2010nonlinear}, PCA is performed on a snapshot of states of the system to find a small dimensional subspace that capture well the dynamics of the system. If the system evolves, then this subspace may be updated using SuperPCA with the new snapshot in place of $X_2$.

\subsubsection{Effect of the signal distribution} 
We first show how SuperPCA performs for different distributions of the signal strengths on three examples. We define the matrix of population principal components $U_1$ as set up above, with $C_{\mathrm{big}}=1000$ and $n_{\mathrm{big}}=20$. In this way the signals are rather close to being sparse, in the sense that most components in $U_1$ are very small compared to the components of the $20$ most important rows. We find that since $20$ rows of $\Uh_1$ are much larger than the rest, subsampling $s=20$ rows of $X_2$ is sufficient and gives the best accuracy in SuperPCA for a given value of $N$. Therefore we choose to subsample $s=2r=20$ rows in SuperPCA.

We experiment with three signal distributions:
\begin{itemize}   
    \item Slowly decaying leading signal strengths (Figure~\ref{fig:closeSigs}): we define $\beta_i = (1000-i^3)/200+1$, such that the three leading signals have nearly the same strength. 

    \item Uniform signal strengths (Figure~\ref{fig:unifSigs}): we define the signal strengths uniformly in the interval $[2,5]$. 

    \item Exponentially decaying signal strengths (Figure~\ref{fig:expSigs}): we define $\beta_i = 10 e^{-i}+1$. 
    
\end{itemize}

In Figure \ref{fig:label1} we can see the gap between the error in the initial guess $\hat u_1$ (blue dotted line) and the best candidate in $\image (\Uh_1)$ (orange line). In the first two cases the gap is considerable and SuperPCA improves the accuracy significantly as $N\to \infty$. Remarkably, for the same number of measurements, SuperPCA is more accurate than RR-PCA and than classical PCA. A consequence is for example that in Figure~\ref{fig:unifSigs} for the same accuracy, SuperPCA requires almost 100$\times$ less measurements than RR-PCA. Finally, in the example with the exponentially decaying signal strengths, the best candidate in $\image (\Uh_1)$ is not much more accurate than the initial guess and the SuperPCA and RR-PCA algorithms do not improve the estimation faster than the classical PCA algorithm. Since the accuracy that SuperPCA and RR-PCA can achieve is limited by the choice of the search space $\Uh_1$, in Figure~\ref{fig:unifSigs} and Figure~\ref{fig:expSigs} the classical PCA algorithm eventually reaches a higher accuracy than the former. In all the cases above, SuperPCA finds the best candidate for $u_1$ in the subspace $\image(\Uh_1)$ after about $N=10^6$ samples.

Hence, comparing the three examples, we see that SuperPCA is particularly effective when the leading signals have about the same strength and the strength of the following components have a fast decay. We can explain this with the same intuition as in the introduction.

In the first case, since $\beta_1$ and $\beta_2$ are close to each other (both approximately equal to $5.9$), we expect that the initial $\hat u_1$ and $\hat u_2$ both contain large components of the exact $u_1$ and $u_2$ (as explained in the introduction), therefore, for $n$ large enough we expect to have a very good best guess of $u_1$ in the subspace $\image(\Uh_1)$ although the initial $\hat u_1$ may be quite inaccurate. 

Conversely, for exponentially decaying signal strengths, the first and second signal have very different strengths ($\beta_1> \beta_2$), so the classical PCA algorithm already approximates the leading principal component very well, and SuperPCA does not significantly improve on $\hat u_1$ when we approximate $u_1$ alone. However if we wished to approximate signals that are closer to the point where the singular values of the covariance matrix flatten out, for example if we estimate $\mbox{span}([u_1,\ldots, u_5])$, then SuperPCA with a candidate subspace of dimension $r>5$ could lead to larger improvements in the accuracy, similarly to the case with slowly decaying signal strengths.

\begin{figure}[ht]
    \centering

    \subfloat[\label{fig:closeSigs}]{%
        \includegraphics[width=0.32\textwidth]{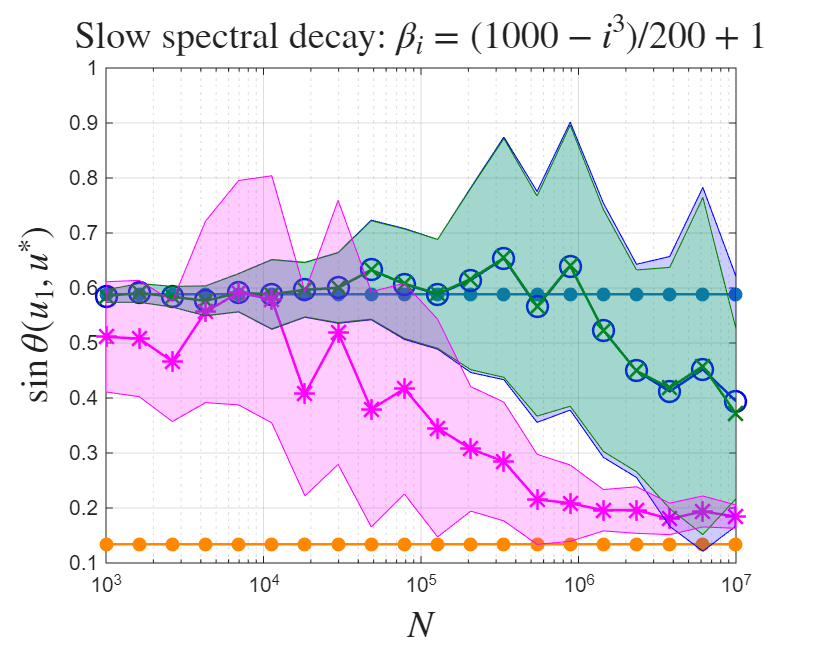}}
    \hfill
    \subfloat[\label{fig:unifSigs}]{%
        \includegraphics[width=0.32\textwidth]{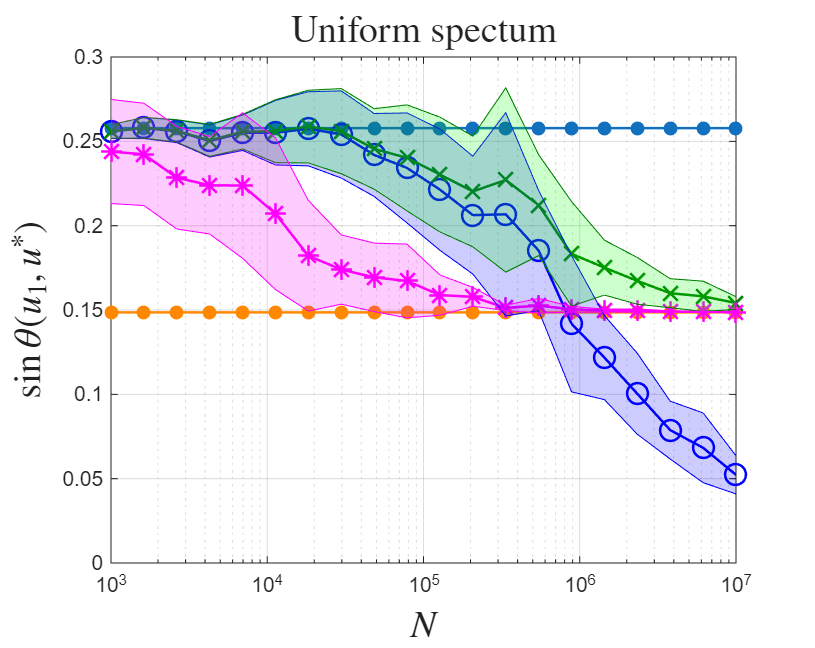}}
    \hfill
    \subfloat[\label{fig:expSigs}]{%
        \includegraphics[width=0.32\textwidth]{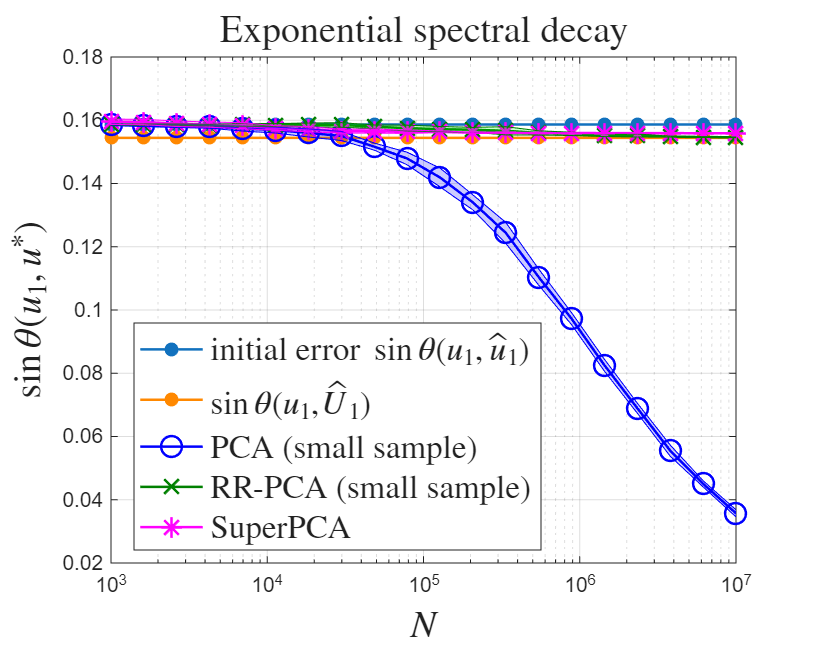}}
     
    \caption{Error in the approximation of the leading principal component $u_1$ for different signal distributions. We average the error over 20 runs and also show the standard deviation for each method.}
    \label{fig:label1}
\end{figure}

It is also worth noting that the RR-PCA method with few data samples is as accurate as classical PCA based on the same $X_2$. Their cost is respectively $\mathcal{O}(nr^2+rnp)$ and $\mathcal{O}(\min(np^2, n^2p))$, so RR-PCA is cheaper and might be a relevant method for subspace tracking in itself.

\subsubsection{Effect of the coherence of $\Uh_1$} 
We now discuss the influence of the coherence of $U_1$ on the accuracy of the SuperPCA algorithm. 
In these experiments, we fix the signal strengths to $\beta_i = (1000-i^3)/200+1, \forall i=1,\ldots,10$, and we define $U_1$ as previously, with $n_{\mathrm{big}}=20$ rows that are weighted by $C_{\mathrm{big}}$. We consider three cases for the value of $C_{\mathrm{big}}$: $1000$, $100$ and $10$, which lead to different values of the coherence of $U_1$. For $C_{\mathrm{big}}=1000$ and $C_{\mathrm{big}}=100$ we sample $s=20$ rows of $X_2$ in SuperPCA, while for $C_{\mathrm{big}}=10$ we choose to sample $s=400$ rows.
The results are presented in Figures \ref{fig:coherence1}, \ref{fig:coherence2} and \ref{fig:coherence3}. In the following discussion, when $C_{\mathrm{big}}\gg 1$ we call the signal \textit{nearly sparse}, in the sense that there are a few dominant coefficients in $U_1$ (and in $\Uh_1$) and the others are near zero.

\begin{figure}[ht]
    \centering

    \subfloat[$C_{\mathrm{big}}=1000, s=20$\label{fig:coherence1}]{%
        \includegraphics[width=0.32\textwidth]{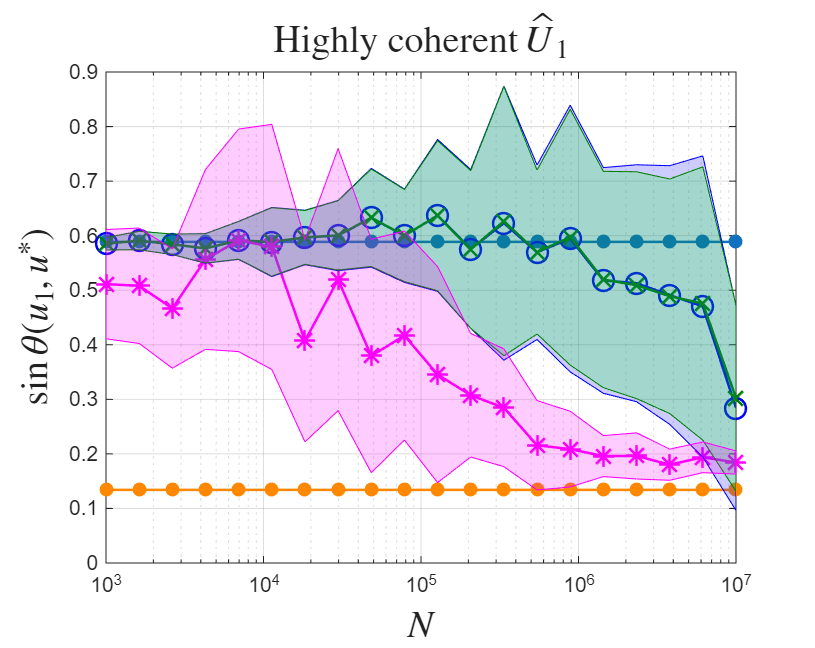}}
    \hfill
    \subfloat[$C_{\mathrm{big}}=100, s=20$\label{fig:coherence2}]{%
        \includegraphics[width=0.32\textwidth]{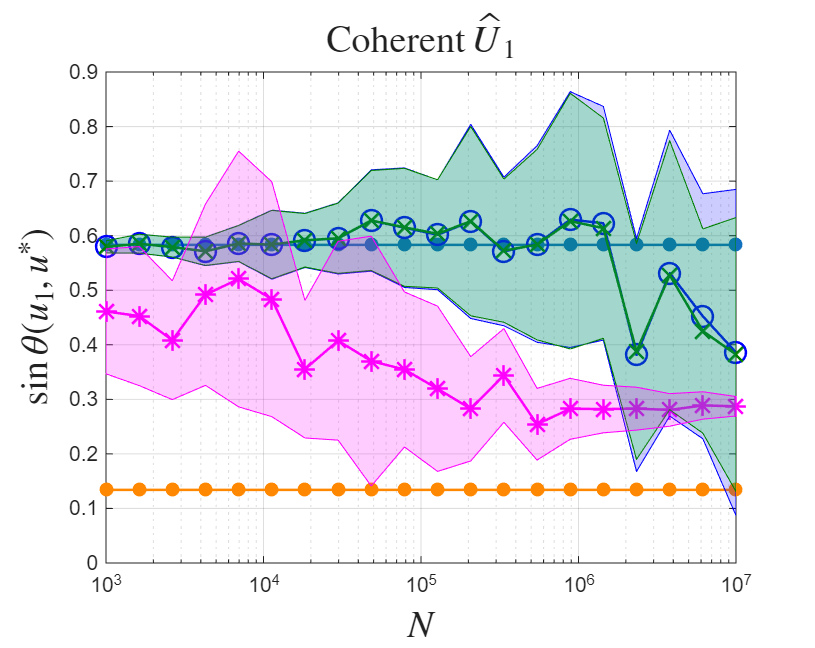}}
    \hfill
    \subfloat[$C_{\mathrm{big}}=10, s=400$\label{fig:coherence3}]{%
        \includegraphics[width=0.32\textwidth]{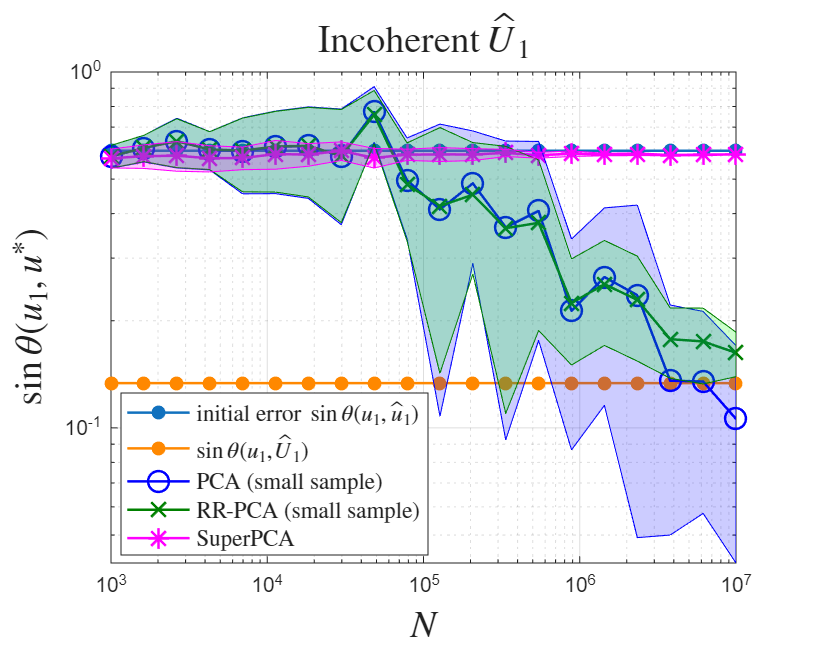}}

    \caption{Error in the approximation of the leading principal component $u_1$ for different structures of the signal. We average the error over 20 runs and also show the standard deviation for each method.}
    \label{fig:closeSigsMod}
\end{figure}

In every figure we see a considerable gap between the accuracy of the initial candidate and that of the best estimate of $u_1$ in $\image (\Uh_1)$. 
In the cases where $C_{\mathrm{big}}=1000$ and $C_{\mathrm{big}}=100$, i.e. in Figures \ref{fig:coherence1} and \ref{fig:coherence2}, the accuracy of SuperPCA improves as $N\to \infty$, coming close to the best accuracy $\sin \theta (u_1,\Uh_1)$, however we see that when $C_{\mathrm{big}}=100$ the accuracy stagnates before reaching $\sin \theta (u_1,\Uh_1)$. In these cases, SuperPCA is much more accurate than RR-PCA and PCA for the same number of measurements. On the other hand, when $C_{\mathrm{big}}=10$, the error in SuperPCA stagnates without improving much compared to the initial guess, while RR-PCA and classical PCA improve the accuracy as $N\to \infty$.

These experiments show that for nearly sparse signal directions SuperPCA improves the accuracy considerably  compared to PCA and RR-PCA for the same number of measurements. Figures \ref{fig:coherence1} and \ref{fig:coherence2} illustrate that the signal is well captured even when subsampling only $s=20$ rows of $X_2$. Remarkably, we observe that for the same accuracy SuperPCA requires ten times fewer measurements than RR-PCA.

On the other hand, when $C_{\mathrm{big}}=10$, the matrix $U_1$ is less coherent, so the error in SuperPCA stagnates and is eventually overtaken by RR-PCA and classical PCA, even if we take $s=400$ in Figure \ref{fig:coherence3}. It is relevant to note that RR-PCA still performs well, which means that $\Uh_1$ is a good candidate subspace that contains a good approximation $u^*\in \image (\Uh_1)$ but this approximation cannot be attained by subsampling only $s=400$ rows. 
Additional experiments indicate that even when only a few rows of $U_1$ are very small, we need to sample all the other rows for SuperPCA to be accurate. A possible explanation for this limitation of the algorithm is that $\sigma_{\min}(S\Uh_1)$ is small if we don't sample all the important rows. For example in the case $C_{\mathrm{big}}=10$, for $s=20$ we get $\sigma_{\min}(S\Uh_1)=0.3$, and for $s=400$ we get approximately $\sigma_{\min}(S\Uh_1)=0.8$. 

However, if the noise level is very small (or equivalently, the signal is very strong) then SuperPCA improves the accuracy of the sample principal component even if the signal directions $U_1$ are incoherent. We show this in Figure \ref{fig:smallNoiseHaar}, where we defined $U_1$ as Haar distributed and the noise level is set to $\sigma=10^{-2}$. It is worth mentioning that for this situation the SuperPCA algorithm performs better without regularization. The L-curve method used in the other experiments over-regularizes the solution, leading to a poor approximation of $u_1$.

Therefore we conclude that, for moderate to weak signals, SuperPCA is very efficient for nearly sparse signal directions (i.e. when $U_1$ has only a few large components and the other components are close to zero), while for very strong signals SuperPCA without regularization is efficient even when the signal directions are dense and incoherent.

\begin{figure}
    \centering
    \includegraphics[width=0.6\linewidth]{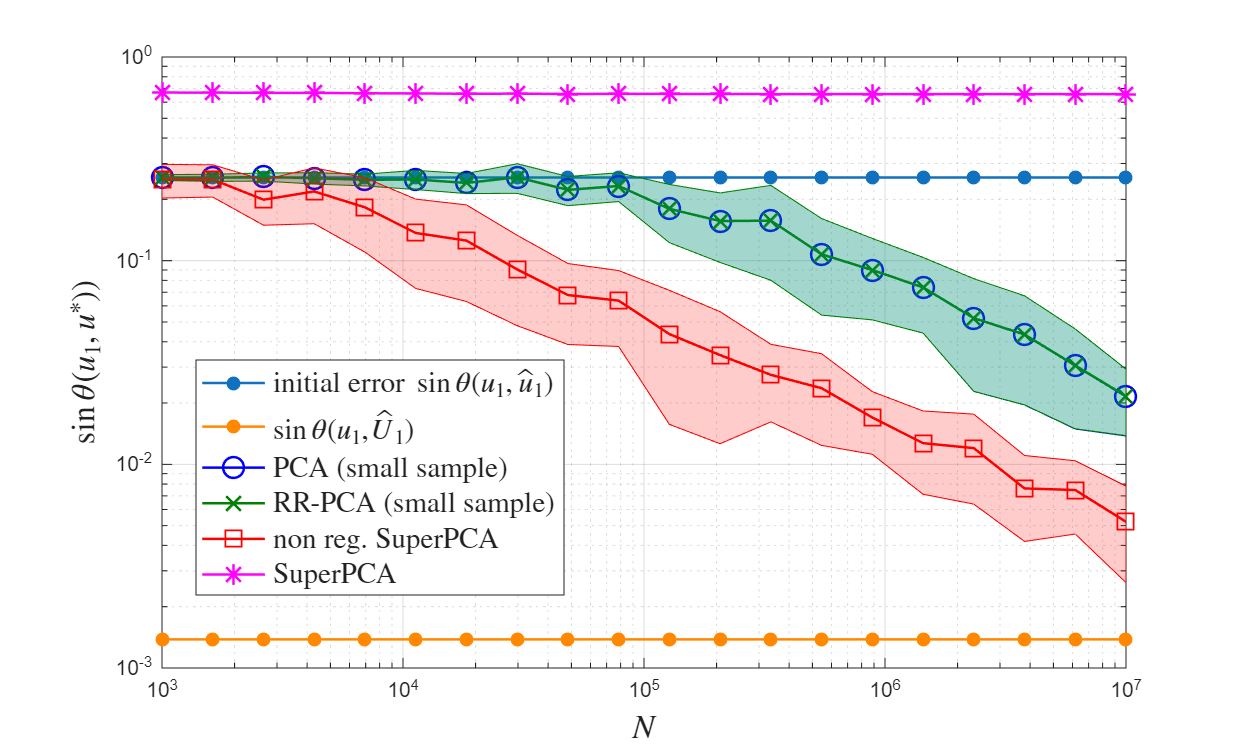}
    \caption{Error for small noise $\sigma=10^{-2}$. We chose uniform signal strengths in the interval $[2,5]$ and the signal directions $U_1$ are Haar distributed (i.e. $C_{\mathrm{big}}=1$). }
    \label{fig:smallNoiseHaar}
\end{figure}

\subsection{Allocation of the available budget in terms of the number of measurements}

The main strength of our method is that once a candidate singular subspace $\Uh_1$ is identified, we do not need to sample the entire data vectors in the follow up data $X_2$, which reduces considerably the cost of the measurements. Therefore one may wonder what portion of the data acquisition budget $B=np+sN$ should be allocated to acquiring $X_1$ and $S X_2$ (the subsampled version of $X_2$) to achieve optimal accuracy.

In Figure \ref{fig:varAlpha1}, we looked at the error in the approximation of $u_1$ over the portion $\alpha:=np/B$ of the budget that is allocated to the initial PCA sample $X_1$ for two values of the fixed total budget $B=10^6, 10^8$, for $\beta_i$ uniformly distributed between $2$ and $5$, then we also included the results obtained with $B=10^8$ and $\beta_i = (1000-i^3)/200+1$. We set $\mbox{dim}(\Uh_1)=10$ and subsample $s=40$ rows of $X_2$ in SuperPCA. 
The exact $U_1$ is defined as set up above with $n_{\mathrm{big}}=20$ rows weighted by $C_{\mathrm{big}}=1000$. For each method for approximating $u_1$, we averaged the error over 50 runs and included the standard deviation.

\begin{figure}[htbp]
    \centering

    \subfloat[$B=10^6$, uniform $\beta_i$\label{fig:budg1}]{%
        \includegraphics[width=0.33\textwidth]{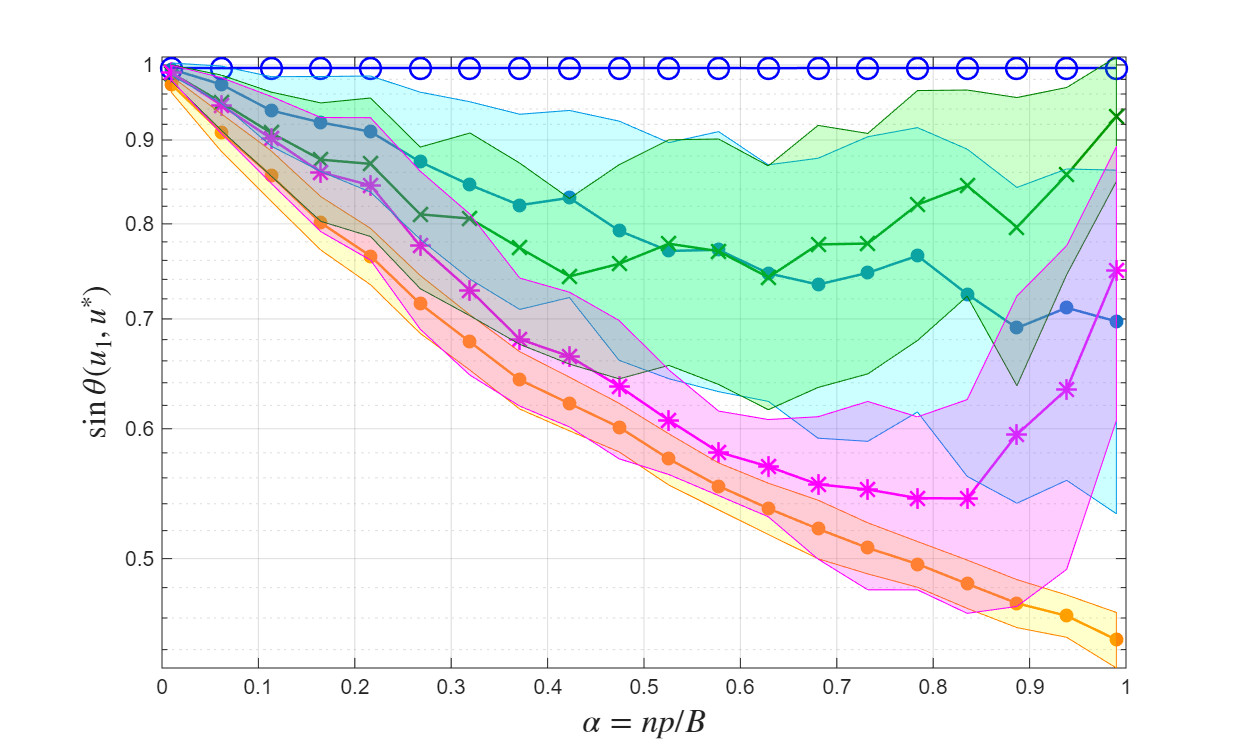}}
    \hfill
    \subfloat[$B=10^8$, uniform $\beta_i$\label{fig:budg2}]{%
        \includegraphics[width=0.33\textwidth]{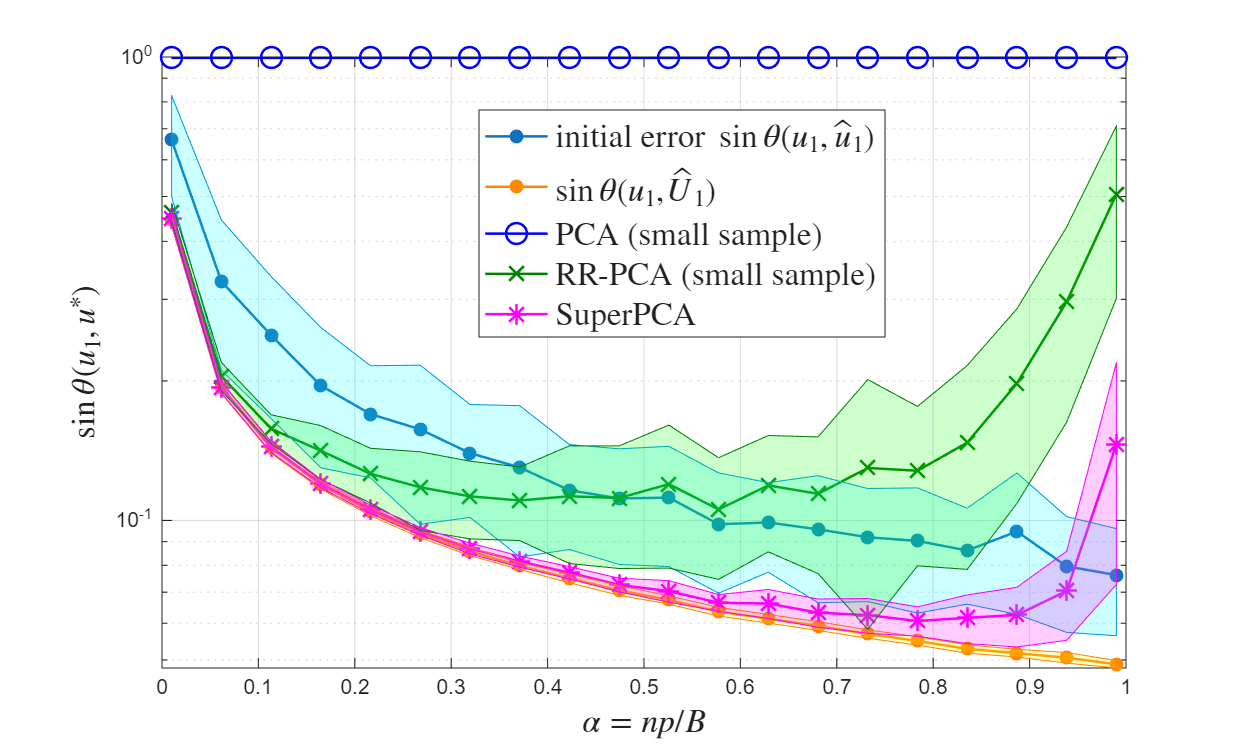}}
    \hfill
    \subfloat[$B=10^{8}, \beta_i = (1000-i^3)/200+1$\label{fig:closeSigs_budg10e8}]{%
        \includegraphics[width=0.33\textwidth]{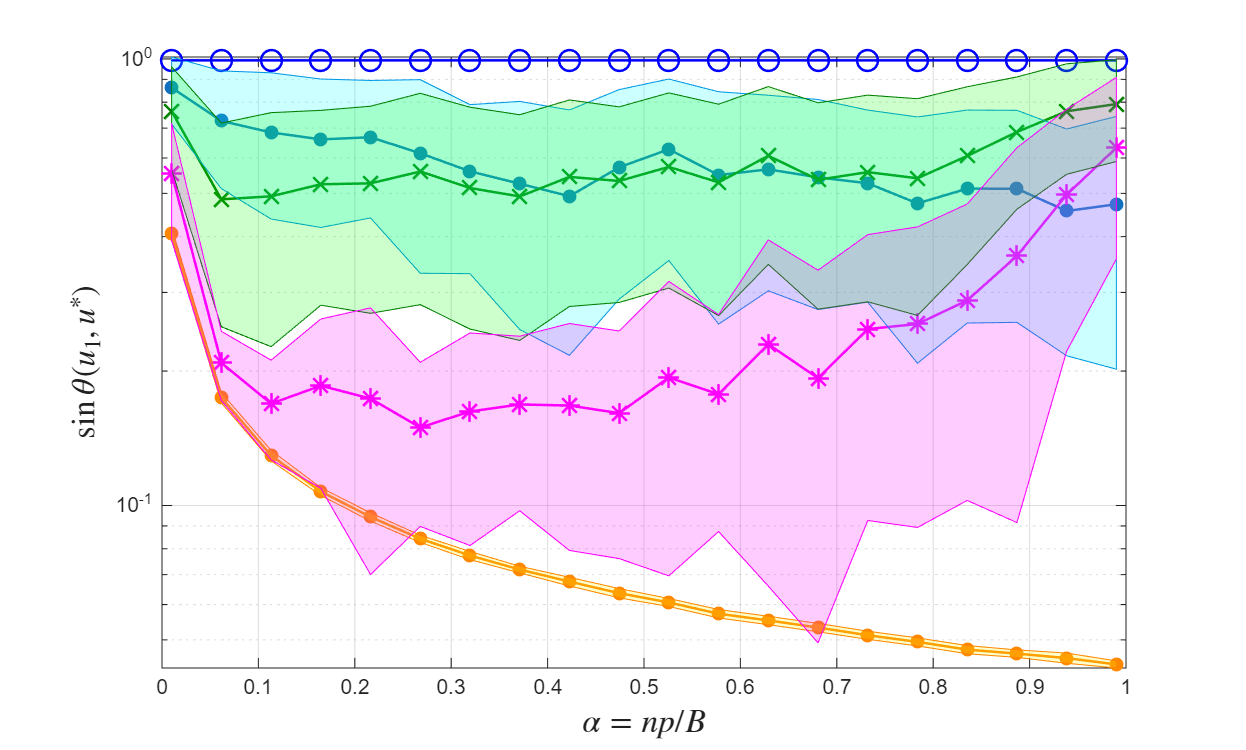}}

    \caption{Error $\sin \theta (u_1,\hat u_1)$ over the portion $\alpha$ of the total budget $B$ allocated to acquiring the data $X_1$ that is used to obtain $\Uh_1$.}
    \label{fig:varAlpha1}
\end{figure}

In Figure~\ref{fig:varAlpha1} we see that for a fixed total budget, the accuracy of the SuperPCA output first improves as $\alpha$ increases, until it reaches a minimum point, then it worsens as $\alpha$ gets closer to 1.

Figures \ref{fig:budg1} and \ref{fig:budg2} show that for uniformly distributed signal strengths the optimal split is to allocate about $70$ to $80\%$ of the measurements to the initial sample $X_1$ and the remainder to acquiring the subsampled $SX_2$ for SuperPCA. This still means that the matrix $SX_2$ is very wide.

However, we observed that the optimal $\alpha$ depends on the distribution of the signal strengths, in particular on the spectral gap $\beta_1-\beta_2$. If the gap is larger the optimal $\alpha$ is closer to 1, i.e. a larger portion of the budget should be invested in the candidate subspace $\Uh_1$, while if the gap is smaller the optimum $\alpha$ is closer to zero. For example in Figure~\ref{fig:closeSigs_budg10e8}, the gap $\beta_1-\beta_2$ is smaller, which gives a smaller value of the optimal $\alpha$ for a total budget $B=10^8$.

\subsection{Application of SuperPCA to MNIST dataset}\label{sec:realApp}
Finally, we applied the SuperPCA algorithm to the MNIST digits dataset \cite{lecun1998gradient}. This dataset contains grayscale images of handwritten numbers from 0 to 9 of size $28 \times 28$ pixels. 
We used the numbers 0, 1 and 2 and defined the data samples by mixing an image of each of these numbers weighted by the strengths 10, 6 and 4 respectively. We generated in this way a matrix $X_{\mathrm{tot}}$ of 1 million samples of dimension $p= 28 ^2 = 784$ and used the leading left singular vectors of this matrix as the population principal components. In this experiment the spectrum of $X_{\mathrm{tot}}$ has a fast decay therefore, as commented in the example with exponentially decaying signal strengths in Section~\ref{subsec:experiments1}, it is a case where it is hard for SuperPCA to do better than classical PCA.

We define the search space $\Uh_1$ by applying classical PCA to 
another set of samples of size $n=50$ and taking its $r=30$ first leading singular vectors. Then we use SuperPCA to approximate the leading 3 principal components, with a follow up data of size $N=233$ generated in the same way as the data samples from $X_{\mathrm{tot}}$. The signals $\Uh_1$ are not very sparse, therefore we choose to subsample $s=10r=300$ rows, which corresponds roughly to the number of rows that have norm above $0.05$. We also use standard PCA with $N_0=89$ samples so that the number of measurements used in SuperPCA and PCA are the same. In addition both in SuperPCA and classical PCA we also used the initial sample of size $n$ in the approximation. Finally, for comparison we also implemented Johnstone and Lu's Sparse PCA algorithm \cite{Johnstone2009}. This algorithm selects the rows of largest variance and applies PCA to the subsampled data then pads the resulting principal components with zeros in the other coordinates. We used the same number of subsampled rows in SuperPCA and in Sparse PCA and the same number of data samples. We selected the rows of largest variance based on the (small) full data sample generated for standard PCA.

This gives the approximate principal components displayed in Figure ~\ref{fig:MNISTimg}. Comparing them with the population principal components we see that the third component is slightly better approximated by SuperPCA than classical PCA. The error $\sin\Theta ([u_1, u_2, u_3],[\hat u_1, \hat u_2, \hat u_3])$ in SuperPCA is 0.35 whereas the error in PCA is 0.58 and the error in Sparse PCA is 0.55. We can see in Figure ~\ref{fig:MNISTimg} that SuperPCA gives a significantly better approximation of the third principal component.

\begin{figure}[ht]
    \centering
    \subfloat[Population principal components\label{fig:top-left}]{%
        \includegraphics[width=0.7\textwidth]{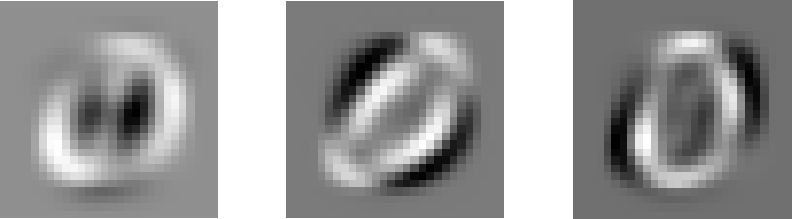}}
    \hfill
    \subfloat[SuperPCA approximation ($N=233$ samples)\label{fig:top-right}]{%
        \includegraphics[width=0.7\textwidth]{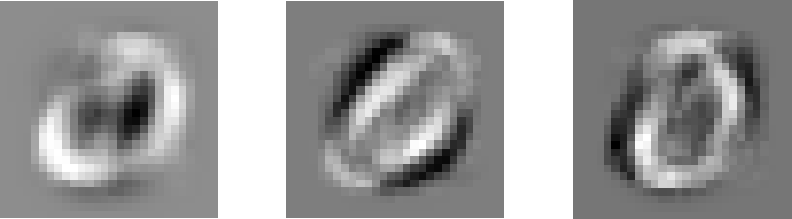}}
    \par\medskip
    \subfloat[PCA approximation ($N_0=89$ samples)\label{fig:bottom}]{%
        \includegraphics[width=0.7\textwidth]{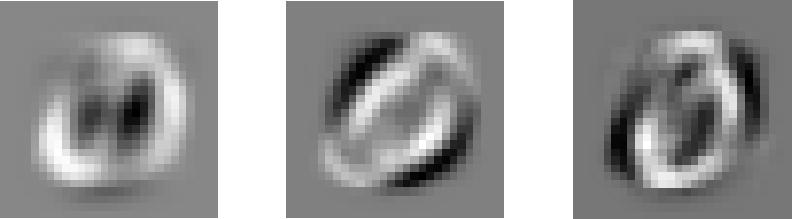}}
    \hfill
    \subfloat[Sparse PCA \cite{Johnstone2009} approximation ($N=233$ samples)\label{fig:bottom2}]{%
        \includegraphics[width=0.7\textwidth]{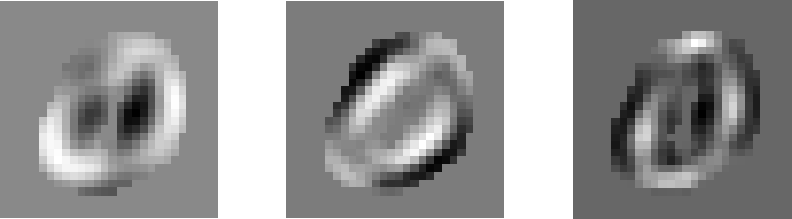}}

    \caption{Leading 3 principal components and their estimates obtained with SuperPCA, PCA and SparsePCA.}
    \label{fig:MNISTimg}
\end{figure}

\section{Conclusion}
We demonstrated that in the spiked covariance model, one can find an estimate of the leading principal components far more accurately than the initial estimate obtained from PCA by exploiting a larger subspace of eigenvectors of the sample covariance matrix. Assuming the leading signals have few dominant components or that the signal is strong, our proposed algorithm, SuperPCA, efficiently finds the best estimate in that subspace. 
An open question is the derivation of bounds on the sine of the angle between the desired principal components and the candidate subspace in terms of $\beta_i$, $n$, and $p$.

Another question is the study of the convergence of the SuperPCA algorithm, or the complexity of the method for a desired final accuracy. For a fixed candidate subspace we observe a $\mathcal{O}(1/\sqrt{N})$ decrease of the sine of the angle between the desired population principal components and the estimated principal components (where $N$ is the number of subsampled data vectors), with a constant factor that probably depends on some spectral gap, but we leave the theoretical proof for future research. 

Additionally, we think that solving a total least square problem that accounts for noise in the matrices $S\Uh_1$ and $SX_2$ could lead to an alternative version of SuperPCA with potential benefits in terms of robustness. We leave this algorithmic idea for future work.

%% file: Proofs.tex
\section{Proofs of Theorems \ref{thm:onesigbound} and \ref{thm:multsigbound}} \label{sec:proofs}

In this section we give proofs of Theorems \ref{thm:onesigbound} and \ref{thm:multsigbound}. Theorem \ref{thm:onesigbound} is specialized to the case when we estimate only the dominant signal, whereas Theorem \ref{thm:multsigbound} considers a subspace spanned by multiple signals. The proof of Theorem \ref{thm:onesigbound} is not a specialization of Theorem \ref{thm:multsigbound} as we show below by proving these results using different techniques. In the special case when $r_1 = 1$, Theorem \ref{thm:multsigbound}'s bound is weaker than Theorem \ref{thm:onesigbound}'s bound when the noise, $\sigma$ is large compared to the signal strengths. 

The statement for Theorem \ref{thm:multsigbound} in this section is a generalization of the statement in Section \ref{sec:theory}. The generalization accommodates cases where the model has significantly smaller noise than the signal strengths. 

\subsection{Proof of Theorem \ref{thm:onesigbound}}
We first transform the scaled data matrix into a block lower triangular form using right orthogonal transformations. Note that right orthogonal transformations do not change the left singular vectors nor the singular values of the matrix. 

First, since Gaussian matrices are invariant under orthogonal transformations, the scaled data matrix can be given by 
\begin{equation}
    X_n = \frac{1}{\sqrt{n}} \left(\sum_{j = 1}^k \sqrt{\beta_j}g_j^i u_j +\sigma \eta_i \right) \overset{d}{=} \frac{1}{\sqrt{n}}[U,U_\perp]\underbrace{\mathrm{diag(\sqrt{\beta_1+\sigma^2},...,\sqrt{\beta_k+\sigma^2},\sigma,...,\sigma)}}_{p\times p} G_{p\times n},
\end{equation} where $U = [u_1,...,u_k]\in \Re^{p\times k}$ is the matrix of signal directions with $U_\perp$ being any orthogonal complement of $U$, and $G_{a\times b}$ is an $a\times b$ Gaussian matrix with entries i.i.d. $\mathcal{N}(0,1)$. Here, $\overset{d}{=}$ is used to denote equality in distribution. Now define
\begin{align} \label{eq:transformData}
    \Xt_n&:= [U,U_\perp]^TX_n \overset{d}{=} \frac{1}{\sqrt{n}}\underbrace{\mathrm{diag(\sqrt{\beta_1+\sigma^2},...,\sqrt{\beta_k+\sigma^2},\sigma,...,\sigma)}}_{p\times p} G_{p\times n} \notag \\
    &=  \frac{1}{\sqrt{n}}\underbrace{\mathrm{diag(\sigma_1,...,\sigma_k,\sigma,...,\sigma)}}_{p\times p} G_{p\times n}. 
\end{align}

For shorthand let the rows of $\Xt_n$ be $b_i = u_i^TX_n \in \R^{1\times n}$ for each $1\leq i \leq k$, i.e., the indices corresponding to the signals. Then
$$ \Xt_n = \begin{bmatrix}
        b_1 \\ \vdots \\ b_k \\ \frac{\sigma }{\sqrt{n}}G_{(p-k)\times n}
    \end{bmatrix},$$
where $b_i \sim \mathcal{N}(0,(\beta_i+\sigma^2)I_n)$. Let $Q_{B_1}$ be an $n\times n$ orthogonal matrix with $b_1/\|b_1\|_2$ as the first column. Then right multiplication with $Q_{B_1}$ sets the first row of $\Xt_n$ to zero, except for the (1,1) element which is set to $\|b_1\|_2$. The rotation does not change the distribution of any of the other rows. Now, we use the notation $b_j^{(-k)}$ to indicate the $(n-k)$-dimensional vector with distribution $\mathcal{N}(0, (\beta_j + \sigma^2)I_{n-k})$. The second row of
$\Xt_n Q_{B_1}$ is $[\sqrt{\beta_2 + \sigma^2}\eta_{21}\,\,\, (b_2^{(-1)})^T]$, for $\eta_{21}\sim\mathcal{N}(0,1)$. We can repeat the same procedure as before, defining an orthogonal matrix $Q_{B_2}=\big[
  \begin{smallmatrix}
1 & 0 \\ 0 &    \hat Q_{B_2} 
  \end{smallmatrix}
\big]$  with the first column equal to $e_1$ (first canonical basis vector) and the second column equal to $ [0 \,\, (b_2^{(-1)})^T]^T$. Continuing this process, we obtain
\begin{equation}\label{eq:XtTriangular}
     \Xt_n Q_{B_1}\dots Q_{B_k} \stackrel{d}{=} \begin{bmatrix}
\|b_1\|_2 & \\
\sqrt{\beta_2+\sigma^2}\eta_{21} & \|b_2^{(-1)}\|_2 \\
\vdots &  & \ddots\\
\sqrt{\beta_k+\sigma^2}\eta_{k1} & \sqrt{\beta_k+\sigma^2}\eta_{k2} & \cdots & \|b_k^{(-(k-1))}\|_2\\
 & \frac{\sigma}{\sqrt{n}} G_{p-k,k} &&& \frac{\sigma}{\sqrt{n}}G_{p-k,n-k}
  \end{bmatrix} := \tilde{X}^{(T)}_n,
\end{equation}
where each $\eta_{ij}\sim \mathcal{N}(0,1)$ is i.i.d., 
$G_{i,j}$ is $i\times j$ Gaussian, and
$b_j^{-(j-1)}\in\mathbb{R}^{n-j+1}$ is a Gaussian vector with distribution $\sim \mathcal{N}(0,(\beta_j+\sigma^2)I_{n-j+1})$.

We now use the block lower triangular form, $\Xt^{(T)}_n$ to prove Theorem \ref{thm:onesigbound}. Note that $\Xt^{(T)}_n$ and $\Xt_n$ have the same left singular vectors and $X_n,\Xt_n$ and $\Xt^{(T)}_n$ all have the same singular values.

\begin{theorem}[Theorem \ref{thm:onesigbound}] \label{thm:genMultsig}
    Let $X_n = \frac{1}{\sqrt{n}}X=\frac{1}{\sqrt{n}}[x_1,x_2,...,x_n] \in \Re^{p\times n}$ be the scaled data matrix where the data matrix $X$ is as defined in \eqref{eq:1.defmodel}. Then
    \begin{equation} \tag{\ref{eq:onesigbound}}
        \sin\theta(u_1,\Uh_1) \leq \sqrt{\frac{\sigma_1(X_n)^2-\norm{u_1^TX_n}_2^2}{\sigma_1(X_n)^2-\sigma_{r_1+1}(X_n)^2}}
    \end{equation} where $u_1$ is the dominant signal direction and $\Uh_1 \in \Re^{p\times r_1}$ is the matrix containing the $r_1$ leading left singular vectors of $X_n$.
\end{theorem}
\begin{proof}
    First, let the SVD of $\Xt_n^{(T)}$ be
    \begin{equation}
        \Xt^{(T)}_n= \Ut \Sigh \Vt^T = \left[\Ut_1,\Ut_2\right]
    \begin{bmatrix}
        \Sigh_1 &  \\  
         & \Sigh_2
    \end{bmatrix}
    \left[\Vt_1,\Vt_2 \right]^T
    \end{equation} where $\Ut_1 \in \Re^{p\times r_1}, \Ut_2 \in \Re^{p\times (p-r_1)}, \Vt_1 \in \Re^{n\times r_1}, \Vt_2 \in \Re^{n\times (n-r_1)}, \Sigh_1 \in \Re^{r_1\times r_1}$ and $\Sigh_2 \in \Re^{(p-r_1)\times (n-r_1)}$. Here, $\Sigh_1$ and $\Sigh_2$ are the singular values of both $\Xt^{(T)}_n$ and $X_n$ where \sloppy $\Sigh_1 = \diag(\sigh_1,\sigh_2,...,\sigh_{r_1})$ and $\Sigh_2 = \begin{bmatrix}
        \diag(\sigh_{r_1+1},\sigh_{r_1+2},...,\sigh_{n}) \\ 0_{(p-n) \times (n-r_1)}
    \end{bmatrix}$ if $p > n$ and $\Sigh_2 = \begin{bmatrix}
        \diag(\sigh_{r_1+1},\sigh_{r_1+2},...,\sigh_{p}) & 0_{(p-r_1)\times (n-p)}
    \end{bmatrix}$ if $p \leq n$, i.e., $\sigh_j = \sigma_j(X_n)$. Furthermore let $\Uh$ be the left singular vectors of $X_n$ and $\Uh_\perp$ a complement of $\Uh$. Note that by letting $\Uh_\perp$ be any orthogonal complement of $\Uh$, we have $[\Ut_1,\Ut_2] = [U,U_\perp]^T[\Uh,\Uh_\perp]$ where $U_\perp$ is any orthogonal complement of $U = [u_1,...,u_k]$. Therefore the left singular vectors of $\Xt_n^{(n)}$ are the left singular vectors of $X_n$ up to left orthogonal transformation by $[U,U_\perp]^T$.

    We now have
    \begin{equation*}
        \sin\theta(u_1,\Uh) = \norm{u_1^T\Uh_\perp}_2 = \norm{e_1^T [u_1, ...,u_k, U_\perp]^T\Uh_\perp}_2 = \sin\theta(e_1,[U,U_\perp]^T\Uh) = \sin\theta(e_1,\Ut_1)
    \end{equation*} where $e_1$ is the first canonical basis vector. Therefore we bound $ \sin\theta(e_1,\Ut_1)$ instead.

    We first bound $\cos\theta(e_1,\Ut_1) = \norm{e_1^T\Ut_1}_2$ using the quantities $e_1^T\Xt^{(T)}_n \Vt_1$ and $e_1^T\Xt^{(T)}_n \Vt_2$. First, for $e_1^T\Xt^{(T)}_n\Vt_1$ we have
    \begin{equation}
        e_1^T\Xt^{(T)}_n \Vt_1 = \norm{b_1}_2 e_1^T\Vt_1 = \norm{b_1}_2 \vt_1
    \end{equation} where $\vt_1 \in \Re^{1\times r_1}$ is the first row of $\Vt_1$. On the other hand, we also have
    \begin{equation}
        e_1^T\Xt^{(T)}_n \Vt_1 = e_1^T \Ut_1 \Sigh_1 = \ut_1 \Sigh_1
    \end{equation} where $\ut_1 \in \Re^{1\times r_1}$ is the first row of $\Ut_1$. Therefore $\norm{b_1}_2 \vt_1 = \ut_1 \Sigh_1$ and we get
    \begin{equation} \label{eq:ineq1}
        \norm{b_1}_2\norm{\vt_1}_2 =\norm{\ut_1 \Sigh_1}_2 \leq \norm{\ut_1}_2  \sigh_{1}.
    \end{equation}
    Similarly for $e_1^T\Xt^{(T)}_n \Vt_2$ we get $\norm{b_1}_2 \vt_2 = \ut_2 \Sigh_2$, where $\ut_2$ is the first row of $\Ut_2$ and $\vt_2$ is the first row of $\Vt_2$, which gives us 
    \begin{equation} \label{eq:ineq2}
       \norm{b_1}_2 \norm{\vt_2}_2 \leq \norm{\ut_2}_2 \sigh_{r_1+1}.
    \end{equation} Now since $\Ut$ and $\Vt$ are orthogonal matrices, we have
    $\norm{\ut_1}_2^2+\norm{\ut_2}_2^2 = 1$ and $\norm{\vt_1}_2^2+\norm{\vt_2}_2^2 = 1$. Therefore we can rewrite \eqref{eq:ineq1} as
    \begin{equation} \label{eq:ineq1r}
        \norm{b_1}_2^2 \left(1-\norm{\vt_2}_2^2\right) \leq  \left(1-\norm{\ut_2}_2^2\right) \sigh_1^2,
    \end{equation} which gives us
    \begin{equation*}
        \norm{\ut_2}_2^2 \leq 1- \frac{\norm{b_1}_2^2}{\sigh_1^2}\left(1-\norm{\vt_2}_2^2\right) \leq 1- \frac{\norm{b_1}_2^2}{\sigh_1^2}+\frac{\norm{\ut_2}_2^2\sigh_{r_1+1}^2}{\sigh_1^2}
    \end{equation*} using \eqref{eq:ineq2} and \eqref{eq:ineq1r}.
    Finally, rearranging gives us the desired inequality,
    \begin{align*}
        \sin\theta(u_1,\Uh) = \sin\theta(e_1,\Ut_1) = \norm{e_1^T\Ut_2}_2 = \norm{\ut_2}_2 \leq \sqrt{\frac{\sigh_1^2-\norm{b_1}_2^2}{\sigh_1^2-\sigh_{r_1+1}^2}} = \sqrt{\frac{\sigma_1(X_n)^2-\norm{u_1^TX_n}_2^2}{\sigma_1(X_n)^2-\sigma_{r_1+1}(X_n)^2}}.
    \end{align*}
\end{proof}

Extending this proof to the case where multiple signals are approximated simultaneously gives a slightly different bound involving the smallest singular value $\tilde \sigma$ of the top left $r_1\times r_1$ lower triangular block of $\Xt^{(T)}_n$. Since this makes the numerator substantially larger, the bound does not scale well, according to our experiments.

\subsection{Proof of Theorem \ref{thm:multsigbound}}
In this section, we prove Theorem \ref{thm:multsigbound} in a more general form, which scales better for small noise. We work with the matrix $\Xt_n = [U,U_\perp]^T X_n$ in this section as in \eqref{eq:XtTriangular}. Let $r_*$ be a parameter defined for analysis and satisfies $r_1 \leq r_* \leq r_2$ (Theorem \ref{thm:multsigbound} corresponds to the case when $r_* = r_1$). Note that the dimensions satisfy the following relation,
\begin{equation*}
    1 \leq r_1 \leq r_* \leq r_2 \leq k \leq \min\{p,n\}.
\end{equation*} 
Now define $\Xt_1 \in \R^{r_* \times n}$ and $\Xt_2 \in \R^{(p-r_*)\times n}$ as
\begin{equation} \label{eq:divDataMat}
        \Xt_n = \begin{bmatrix} U & U_\perp\end{bmatrix}^T X_n = \frac{1}{\sqrt{n}}\begin{bmatrix}
        b_1 \\
    \vdots \\
    b_k \\
    \sigma G_{(p-k)\times n}
    \end{bmatrix} = \begin{bmatrix}
        \Xt_1 \\ \Xt_2 
    \end{bmatrix},
\end{equation} where $b_i \sim \mathcal{N}(0,(\beta_i+\sigma^2)I_n)$. We use this division of the transformed data matrix $\Xt_n$ in the proof.

For subspaces of equal dimension, an important property of the canonical angles is that $\norm{\sin\Theta(\cdot,\cdot)}$ is a metric on the space of $k$-dimensional subspaces \cite{MPT1990}. We extend the corresponding triangle inequality to subspaces of different dimensions in the Lemma \ref{lem:triangle} below.

\begin{lemma}[Triangle inequality for subspaces of different dimensions] \label{lem:triangle}
Let $U = [U_1,U_2,U_3]$, $\hat{U} = [\hat{U}_1,\hat{U}_2,\hat{U}_3]$ and $\tilde{U} = [\tilde{U}_1,\tilde{U}_2,\tilde{U}_3]$ be $n\times n$ orthogonal matrices where index $1$ corresponds to $j$ columns, index $2$ corresponds to $\ell-j>0$ columns and index $3$ corresponds to the other $n-\ell$ columns. Then 
    \begin{equation}
        \norm{\sin\Theta(U_1,[\tilde{U}_1,\tilde{U}_2])} \leq \norm{\sin\Theta(U_1,[\hat{U}_1,\hat{U}_2])} +\norm{\sin\Theta([\hat{U}_1,\hat{U}_2],[\tilde{U}_1,\tilde{U}_2])}
    \end{equation}
    and 
    \begin{equation}
        \norm{\sin\Theta(U_1,[\tilde{U}_1,\tilde{U}_2])} \leq \norm{\sin\Theta(U_1,\hat{U}_1)} +\norm{\sin\Theta(\hat{U}_1,[\tilde{U}_1,\tilde{U}_2])}
    \end{equation} for any unitarily invariant norms.
\end{lemma}
\begin{proof}
    \begin{align*}
    \norm{\sin\Theta(U_1,[\tilde{U}_1,\tilde{U}_2])} = \norm{U_1^T\tilde{U}_3} &=\norm{U_1^T[\hat{U}_1,\hat{U}_2,\hat{U}_3] [\hat{U}_1,\hat{U}_2,\hat{U}_3]^T\tilde{U}_3 } \\
    &= \norm{U_1^T[\hat{U}_1, \hat{U}_2][\hat{U}_1, \hat{U}_2]^T\tilde{U}_3} +\norm{U_1^T\hat{U}_3\hat{U}_3^T \tilde{U}_3} \\
    &\leq \norm{[\hat{U}_1, \hat{U}_2]^T\tilde{U}_3} +\norm{U_1^T\hat{U}_3} \\
    &\leq \norm{\sin\Theta([\hat{U}_1,\hat{U}_2],[\tilde{U}_1,\tilde{U}_2])} + \norm{\sin\Theta(U_1,[\hat{U}_1,\hat{U}_2])}.
    \end{align*} The second result can be proved similarly.
\end{proof}

Lemma \ref{lem:triangle} will be useful for proving a generalization of Theorem \ref{thm:multsigbound} below as we compare subspaces of different dimensions.

\begin{theorem}[Generalization of Theorem \ref{thm:multsigbound}]\label{thm:multsigbound2}
    Let $X_n =\frac{1}{\sqrt{n}}X = \frac{1}{\sqrt{n}} [x_1,x_2,...,x_n]\in \R^{p\times n}$ be the scaled data matrix where the data matrix $X$ is as defined in \eqref{eq:1.defmodel}.
    Then 
    \begin{equation}
        \norm{\sin\Theta(U_1,[\Uh_1,\Uh_2])}_2 \leq \min_{r_*\in [r_1,r_2]}\frac{\norm{U_{*,\perp}^T X_n \Vc}_2 \sigma_{\min}(U_*^TX_n)}{
        \sigma_{\min}(U_*^TX_n)^2 - \sigma_{r_2+1}(X_n)^2},
    \end{equation} where $U_*$ is the $r_*$ leading signal directions with $U_{*,\perp}$ as its orthogonal complement, $[\Uh_1,\Uh_2]$ is the leading $r_2$ left singular values of the data matrix and $\Vc$ is the leading $r_1$ right singular vectors of $U_1^TX_n$, which is independent of $U_{*,\perp}^TX_n$.
\end{theorem}
\begin{proof}
    Let $\Uh_\perp$ be an orthogonal complement of $[\Uh_1,\Uh_2]$ and $[\Ut_1,\Ut_2,\Ut_\perp] \in \Re^{p\times p}$ be the orthogonal matrix of left singular vectors of $\Xt_n$ where $\Ut_1\in \Re^{p\times r_1}$, $\Ut_2 \in \Re^{p\times (r_2-r_1)}$ and $\Ut_\perp \in \Re^{p\times (p-r_2)}$. Note that $[\Ut_1,\Ut_2,\Ut_\perp] = [U_1,U_\perp]^T[\Uh_1,\Uh_2,\Uh_\perp]$. Then 
    \begin{align*}
        \norm{\sin\Theta(U_1,[\Uh_1,\Uh_2])}_2 &= \norm{U_1^T\Uh_\perp}_2 = \norm{[I_{r_1},0][U_1,U_\perp]^T\Uh_\perp}_2 \\ &= \norm{[I_{r_1},0]\Ut_\perp}_2 =\norm{\sin\Theta\left(\begin{bmatrix}
            I_{r_1} \\ 0
        \end{bmatrix},[\Ut_1,\Ut_2]\right)}.
    \end{align*}
 Let the SVD of $\Xt_1 \in \Re^{r_*\times n}$ be $\Xt_1 = \Uc \Sigc \Vc^T$ and let $\Vc_\perp \in \R^{n\times (n-r_*)}$ be an orthogonal complement of $\Vc$. Now, form the following matrix,
\begin{equation}
    \Xc_n = \begin{bmatrix}
        \Uc & 0 \\
       0  &  I_{(p-r_*)}
    \end{bmatrix}^T \begin{bmatrix}
        \Xt_1 \\ \Xt_2
    \end{bmatrix}
    \begin{bmatrix}
        \Vc & \Vc_{\perp}
    \end{bmatrix} = \begin{bmatrix}
        \Sigc  & 0 \\
        \Xt_2\Vc  & \Xt_2\Vc_{\perp}
    \end{bmatrix} \in \R^{p\times n}.
\end{equation} Note that $\Vc$ and $\Vc_{\perp}$ are only dependent on $\Xt_1$ and therefore $\Vc$ and $\Vc_{\perp}$ are independent of $\Xt_2$ as $\Xt_1$ and $\Xt_2$ are independent. We also note that the singular values of $\Xc_n$, $\Xt_n$ and $X_n$ are the same because $\Xc_n$ and $\Xt_n$ are orthogonal transformations of $X_n$. 

Now let us note that
\begin{align*}
    \sin\Theta\left(\begin{bmatrix}
        \Uc \\ 0
    \end{bmatrix}, \begin{bmatrix}
        \Ut_1,\Ut_2
    \end{bmatrix}\right) &= \sin\Theta\left(\begin{bmatrix}
        \Uc & 0 \\
       0  &  I_{(p-r_*)}
    \end{bmatrix}^T\begin{bmatrix}
        \Uc \\ 0
    \end{bmatrix},\begin{bmatrix}
        \Uc & 0 \\
       0  &  I_{(p-r_*)}
    \end{bmatrix}^T\begin{bmatrix}
        \Ut_1,\Ut_2
    \end{bmatrix}\right) \\
    &= \sin\Theta\left(\begin{bmatrix}
        I_{r_*} \\ 0
    \end{bmatrix}, \begin{bmatrix}
        \Uc_1,\Uc_2
    \end{bmatrix}\right)
\end{align*} where $\Uc_1$ is the leading $r_1$ left singular vectors of $\Xc_n$ and $\Uc_2$ is the next $(r_2-r_1)$ leading left singular vectors. Now, using the triangle inequality, Lemma \ref{lem:triangle}, we obtain 
\begin{align*}
    \norm{\sin\theta\left(\begin{bmatrix}
        I_{r_1} \\ 0
    \end{bmatrix},[\Ut_1,\Ut_2]\right)}_2 &\leq \norm{\sin\theta\left(\begin{bmatrix}
        I_{r_1} \\ 0
    \end{bmatrix},\begin{bmatrix}
        \Uc \\ 0
    \end{bmatrix}\right)}_2 +\norm{\sin\theta\left(\begin{bmatrix}
        \Uc \\ 0
    \end{bmatrix}, \begin{bmatrix}
        \Ut_1,\Ut_2
    \end{bmatrix}\right)}_2 \\ &=
    \norm{\sin\Theta\left(\begin{bmatrix}
        I_{r_*} \\ 0
    \end{bmatrix}, \begin{bmatrix}
        \Uc_1,\Uc_2
    \end{bmatrix}\right)}_2
\end{align*} since 
\begin{equation}
    \norm{\sin\theta\left(\begin{bmatrix}
        I_{r_1} \\ 0
    \end{bmatrix},\begin{bmatrix}
        \Uc \\ 0
    \end{bmatrix}\right)}_2 = \norm{\begin{bmatrix}
        I_{r_1} \\ 0
    \end{bmatrix}^T \begin{bmatrix}
        0 \\ I_{p-r_*}
    \end{bmatrix}}_2 = 0
\end{equation} because $\Uc \in \Re^{r_*\times r_*}$ is an orthogonal matrix and $r_1\leq r_*$.
Therefore, it suffices to bound
\begin{equation}
    \norm{\sin\theta\left(\begin{bmatrix}
        I_{r_*} \\ 0
    \end{bmatrix}, \begin{bmatrix}
        \Uc_1,\Uc_2
    \end{bmatrix}\right)}_2 = \norm{\begin{bmatrix}
        I_{r_*} \\ 0
    \end{bmatrix}^T\Uc_{3}}_2 = \norm{\Uc_{31}}_2
\end{equation} instead, where $\Uc_{3}\in \Re^{p \times (p-r_2)}$ is the trailing $(p-r_2)$ left singular vectors of $\Xc_n$ and $\Uc_{31}$ is the first $r_*$ rows of $\Uc_3$. 

We now analyze $\norm{\Uc_{31}}_2$ using a similar technique to one used in \cite{Nakatsukasa2020b}. Let $\Vc_3 \in \Re^{n\times (n-r_2)}$ be the trailing right singular vectors of $\Xc$ and $\Sigh_3 = \begin{bmatrix}
        \diag(\sigh_{r_2+1},\sigh_{r_2+2},...,\sigh_{n}) \\ 0_{(p-n) \times (n-r_2)}
    \end{bmatrix}$ if $p > n$ or $\Sigh_3 = \begin{bmatrix}
        \diag(\sigh_{r_2+1},\sigh_{r_2+2},...,\sigh_{p}) & 0_{(p-r_2)\times (n-p)}
    \end{bmatrix}$ if $p \leq n$ where $\sigh_i$'s are the singular values of $X_n$, since $X_n$ and $\Xc_n$ have the same singular values. Then we have $\Xc_n \Vc_3 = \Uc_3 \Sigh_3$ and $\Xc_n^T \Uc_3 = \Vc_3 \Sigh_3$. Now divide $\Uc_3$ and $\Vc_3$ into blocks as
    \begin{equation}
        \Uc_3 = \begin{bmatrix}
            \Uc_{31} \\ \Uc_{32}
        \end{bmatrix}, \hspace{1cm} \Vc_3 = \begin{bmatrix}
            \Vc_{31} \\ \Vc_{32}
        \end{bmatrix}
    \end{equation} where $\Uc_{31}\in \Re^{r_*\times (p-r_2)},\Uc_{32}\in \Re^{(p-r_*)\times (p-r_2)},\Vc_{31}\in \Re^{r_*\times (n-r_2)}$ and $\Vc_{32}\in \Re^{(n-r_*)\times (n-r_2)}$. Then using $\Xc_n \Vc_3 = \Uc_3 \Sigh_3$, we get 
    \begin{equation} \label{eq:multsig1}
        \begin{bmatrix}
        \Sigc  & 0 \\
        \Xt_2\Vc  & \Xt_2\Vc_{\perp}
        \end{bmatrix} \begin{bmatrix}
            \Vc_{31} \\ \Vc_{32}
        \end{bmatrix} = \begin{bmatrix}
            \Uc_{31} \\ \Uc_{32} 
        \end{bmatrix} \Sigh_3
    \end{equation} and using $\Xc_n^T \Uc_3 = \Vc_3 \Sigh_3$, we get
    \begin{equation} \label{eq:multsig2}
        \begin{bmatrix}
         \Sigc  & ( \Xt_2\Vc )^T\\
        0 & (\Xt_2\Vc_{\perp})^T
        \end{bmatrix}\begin{bmatrix}
            \Uc_{31} \\ \Uc_{32}
        \end{bmatrix}  = \begin{bmatrix}
            \Vc_{31} \\ \Vc_{32}
        \end{bmatrix} \Sigh_3.
    \end{equation} The first block of \eqref{eq:multsig1} and \eqref{eq:multsig2} give
    \begin{equation}
        \Sigc \Vc_{31} = \Uc_{31}\Sigh_3, \hspace{1cm}
        \Sigc \Uc_{31} + (\Xt_2\Vc)^T \Uc_{32}= \Vc_{31}\Sigh_3
    \end{equation} from which we get the following inequalities
    \begin{equation} \label{eq:multsigineq1}
        \sigma_{\min}(\Sigc) \norm{\Vc_{31}}_2 \leq \norm{ \Sigc \Vc_{31}}_2 = \norm{\Uc_{31}\Sigh_3}_2 \leq  \norm{\Uc_{31}}_2\norm{\Sigh_3}_2
    \end{equation} and 
    \begin{equation} \label{eq:multsigineq2}
        \sigma_{\min}(\Sigc) \norm{\Uc_{31}}_2 \leq \norm{\Sigc \Uc_{31}}_2 = \norm{\Vc_{31}\Sigh_3-(\Xt_2\Vc)^T \Uc_{32}}_2 \leq  \norm{\Vc_{31}}_2\norm{\Sigh_3}_2 + \norm{(\Xt_2\Vc)^T}_2.
    \end{equation}
    Now multiply \eqref{eq:multsigineq1} by $\norm{\Sigh_3}_2$ and multiply \eqref{eq:multsigineq2} by $\sigma_{\min}(\Sigc)$ and then add them together to obtain
    \begin{align*}
        \norm{\Uc_{31}}_2 \leq \frac{\norm{\Xt_2\Vc}_2 \sigma_{\min}(\Sigc)}{(\sigma_{\min}(\Sigc))^2-\norm{\Sigh_3}^2_2} = \frac{\norm{\Xt_2 \Vc}_2 \sigma_{\min}(\Xt_1)}{
        \sigma_{\min}(\Xt_1)^2 - \sigma_{r_2+1}(X_n)^2} 
    \end{align*}
Therefore 
\begin{equation} 
     \norm{\sin\Theta(U_1,[\Uh_1,\Uh_2])}_2 \leq \frac{\norm{\Xt_2 \Vc}_2 \sigma_{\min}(\Xt_1)}{
        \sigma_{\min}(\Xt_1)^2 - \sigma_{r_2+1}(X_n)^2} = \frac{\norm{U_{*,\perp}^T X_n \Vc}_2 \sigma_{\min}(U_*^TX_n)}{
        \sigma_{\min}(U_*^TX_n)^2 - \sigma_{r_2+1}(X_n)^2}
\end{equation} for any $r_*\in [r_1,r_2]$. Minimizing with respect to $r_* \in [r_1,r_2]$, we get the desired result.
\end{proof}

Theorem \ref{thm:multsigbound2} is a generalization of Theorem \ref{thm:multsigbound}. Theorem \ref{thm:multsigbound} corresponds to the case when $r_* = r_1$, which scales well when the noise level is high compared to the signal strengths. When the signal strengths are much larger than the noise level, $r_* = r_1$, hence Theorem \ref{thm:multsigbound} scales badly. Generally, taking $r_* = r_2$ scales much better in this case and Theorem \ref{thm:genMultsig} accommodates both cases.